\documentclass[letterpaper, 10 pt, conference]{ieeeconf}  

\IEEEoverridecommandlockouts                              

\usepackage{graphics} 
\usepackage{times} 

\usepackage{amsmath} 
\usepackage{amssymb}  
\usepackage{amsthm}
\usepackage{mathtools,amsfonts,amssymb,bm}
\usepackage{hyperref}
\usepackage[demo]{graphicx}
\usepackage{subcaption}
\usepackage{multirow}
\usepackage{xcolor}
\usepackage{soul}
\usepackage{cite}
\usepackage{float}
\usepackage{algorithm}
\usepackage{algpseudocode}
\usepackage{algorithmicx}
\usepackage{optidef}
\usepackage{anyfontsize}
\usepackage{booktabs}
\usepackage{comment}
\usepackage{microtype}
\usepackage{todonotes}
\newtheorem{definition}{Definition}

\usepackage{optidef}
\newtheorem{theorem}{Theorem}
\usepackage{rotating}
\usepackage{pifont}
\usepackage{setspace}
\usepackage[table]{xcolor}
\newtheorem{proposition}{Proposition}
\definecolor{shadowColor}{gray}{0.9}

\renewenvironment{proof}[1][Proof]{%
  \par\noindent\pushQED{\qed}\normalfont
  \textit{#1.}\enspace\ignorespaces
}{%
  \popQED\par
}
\usepackage{rotating}

\title{\LARGE \bf
CoCoNav: Conformal Control for Safe Robot Navigation in Crowds
}

\author{Cheng Guo$^{1,2}$, Mingzhe Ni$^{1}$ Zheng Liang$^{3}$, Yihu Ling$^{3}$, Yuan Hu$^{3}$, Michele Caprio$^{4}$, Daniele Pucci$^{1,5}$, Wei Pan$^{6}$
\thanks{$^{1}$Department of Computer Science, The University of Manchester, Manchester, UK.}%
\thanks{$^{2}$Human-Robot Interfaces and Interaction Laboratory, Italian Institute of Technology, Genoa, Italy.}%
\thanks{$^{3}$Genisom AI, Shanghai, China.}
\thanks{$^{4}$Department of Computer Science, University of Warwick, Coventry, UK.}
\thanks{$^{5}$Generative Bionics, Genoa, Italy.}
\thanks{$^{6}$School of Engineering, Newcastle University, Newcastle, UK.}
}

\begin{document}
\setstretch{0.95}
\maketitle
\thispagestyle{empty}
\pagestyle{empty}
\begin{abstract}
Safe and efficient robot navigation in crowds requires anticipating pedestrian motion despite uncertain and potentially shifting prediction errors. Existing reactive methods can produce oscillatory behavior, while predictive planners often treat forecasts as exact or rely on restrictive error models. Incorporating conservative uncertainty sets as hard constraints can also render model predictive control (MPC) infeasible. We propose \textit{CoCoNav}, a crowd-navigation framework that combines online conformal calibration with runtime-certified planning. A horizon-specific conformal proportional--integral controller adapts trajectory-error bounds to regulate long-run empirical coverage, enabling the framework to respond to changing prediction errors. A \textit{relax-then-verify} planner preserves solver feasibility by generating nominal trajectories with soft-constrained MPC and separately certifying them, together with contingency maneuvers, against the calibrated bounds before execution. Simulations and quadruped experiments show that CoCoNav achieves a favorable balance among collision avoidance, task success, and navigation efficiency relative to the evaluated baselines.
\end{abstract}

\section{INTRODUCTION}
\label{sec:introduction}
Safe and efficient robot navigation in crowds requires more than avoiding people at their currently observed positions: a robot must anticipate how pedestrians may move and account for the uncertainty of those predictions. This is difficult because human motion is interactive, multimodal, and nonstationary, so prediction errors can change as the crowd configuration and operating environment evolve \cite{francis2025principles}.
The central challenge is therefore to convert fallible trajectory forecasts into actions that make progress without silently losing safety when the forecasts are inaccurate or the planning problem becomes difficult.

Reactive navigation methods compute collision-avoidance actions from the current configuration and remain attractive for their simplicity. Because they do not explicitly reason over future interactions, however, they may react late or behave conservatively in dense flows. Prediction-aware planners instead incorporate forecast pedestrian trajectories into trajectory optimization or model predictive control (MPC), allowing the robot to anticipate crossings and negotiate space over a planning horizon \cite{brito2019model, poddar2023crowd}. Their effectiveness then depends not
only on forecast accuracy, but also on how forecast errors are represented and
used by the planner.

\begin{figure}[t]
    \centering
    \includegraphics[width=\linewidth]{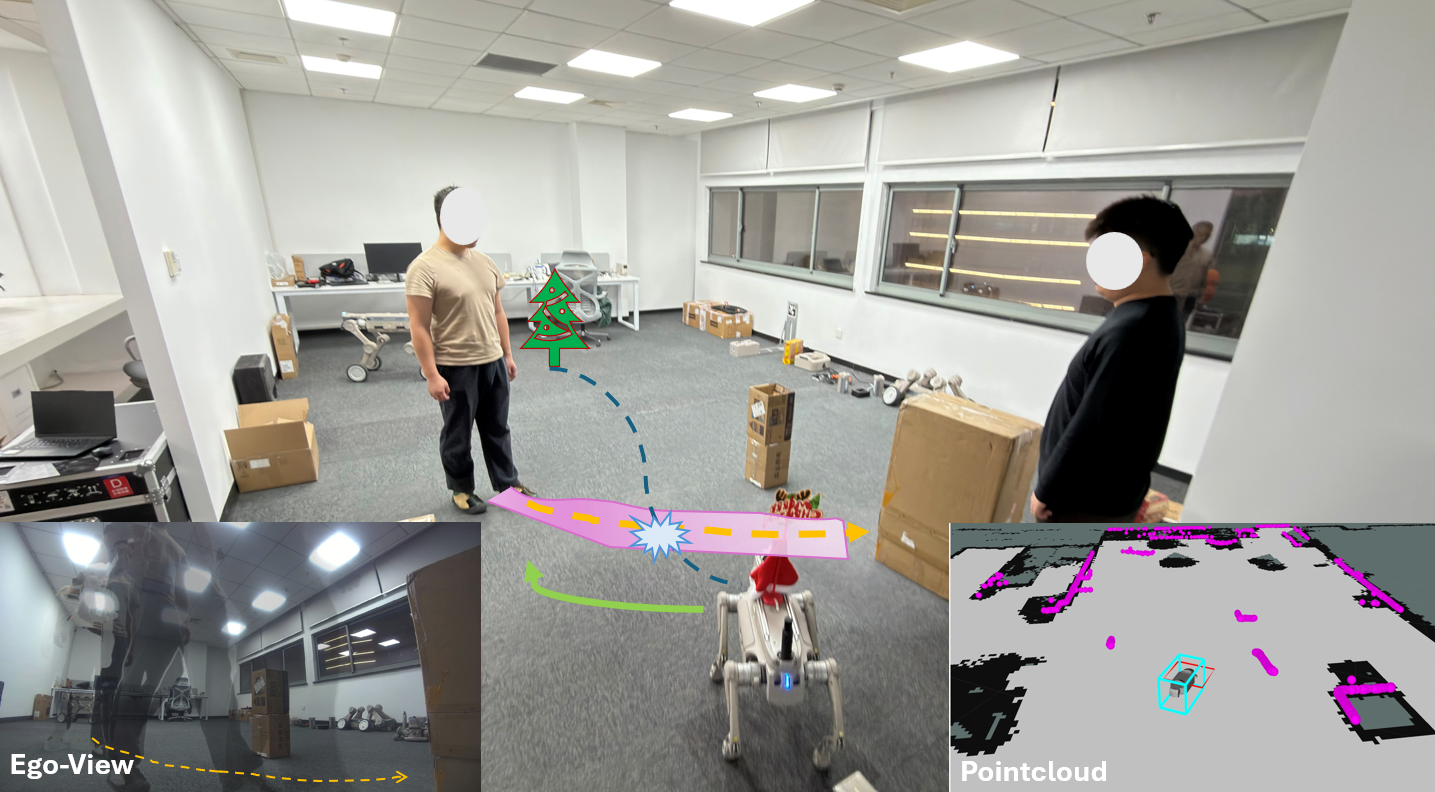}
    \caption{
    Safe navigation around dynamic agents. (\textbf{Bottom Left}) Robot RGB ego-view showing the agent's directional motion trail. (\textbf{Bottom Right}) LiDAR point cloud. (\textbf{Main View}) The robot tracks the global reference (dotted blue) but deviates to a safe path (solid green) to avoid the agent's predicted trajectory (yellow) and conformal uncertainty sets (purple).
    }
    \label{fig:first_page_fig}
\end{figure}

Prior work propagates prediction uncertainty through confidence-aware collision avoidance or stochastic MPC, typically using a specified probabilistic model of future motion
\cite{fridovich2020confidence, nair2022stochastic}. These methods are effective when that model remains representative, but model misspecification and distribution shift can invalidate the resulting uncertainty regions.
Conformal prediction instead calibrates prediction regions from observed errors without prescribing a parametric error distribution \cite{angelopoulos2023conformal}. Adaptive conformal control methods further update these regions online as errors arrive \cite{gibbs2021adaptive, angelopoulos2023conformalpid}. Nevertheless, calibration only determines how much uncertainty to associate with a forecast; it does not determine how the robot should act when a calibrated region is too restrictive for the planner.

This planning challenge motivates a safety-filtering perspective that separates performance-oriented planning from runtime certification. A nominal controller first proposes an action or rollout, which is executed only if a monitoring condition certifies its safety; otherwise, the filter selects a contingency action \cite{hsu2023safety,wabersich2021predictive}. Existing predictive safety filters typically derive guarantees from known uncertainty models, robust reachable sets, or invariant terminal conditions. These assumptions are difficult to satisfy in crowd navigation because pedestrian prediction errors can evolve online with the surrounding scene. A practical filter must therefore distinguish a genuinely certified contingency from a fallback that is merely presumed safe, while also explicitly identifying cases in which no available candidate can be certified.

We address these coupled calibration, feasibility, and certification
challenges with \emph{CoCoNav}. Figure \ref{fig:first_page_fig} illustrates how its calibrated prediction bounds inform crowd-avoidance behavior. 
First, building on the proportional--integral mechanism of \cite{angelopoulos2023conformalpid}, we develop an online calibration scheme for pedestrian trajectory prediction. Each prediction horizon maintains an independent calibration state, and delayed prediction errors are evaluated against the thresholds associated with the forecasts that produced them. Second, a soft-constrained MPC incorporates the calibrated bounds as penalties to generate a progress-oriented nominal rollout, without interpreting the resulting solution as certified safe. Third, an \emph{a posteriori} monitor certifies the entire nominal rollout against the same bounds. If certification fails, the monitor evaluates a library of dynamically admissible contingency trajectories and, if necessary, an acceleration-limited braking rollout before executing the first control input and replanning. By separating candidate generation from safety certification, CoCoNav avoids embedding conformal sets directly as hard MPC constraints. Moreover, rather than assuming the existence of an always-safe fallback, it explicitly records cases in which no available action can be certified. We integrate these components into an onboard pipeline and evaluate the resulting system in simulation and on a quadruped robot, providing initial evidence of its practical viability beyond simulation.
Our main contributions are listed as follows:
\begin{itemize}
    \item We adapt conformal proportional--integral (CPI) control to enable causal, horizon-specific calibration of multi-step pedestrian trajectory errors. The resulting uncertainty bounds adapt to evolving prediction errors while regulating long-run empirical coverage.
    \item We introduce a \emph{relax-then-verify} control scheme in which a
    soft-constrained MPC generates candidate rollouts and an \emph{a
    posteriori} module certifies nominal and contingency rollouts, mitigating
    hard-constraint infeasibility while retaining explicit statistical safety
    accounting.
    \item We evaluate the integrated system in extensive simulations and
    quadruped-robot experiments, providing evidence of its practical
    deployability and its safety--efficiency trade-off.
\end{itemize}

\section{RELATED WORK}
\label{sec:background}
\subsection{Prediction-Aware Crowd Navigation}
\label{sec:crowd_navigation}
\begin{figure*}[th]
    \centering
    \includegraphics[width=0.95\textwidth]{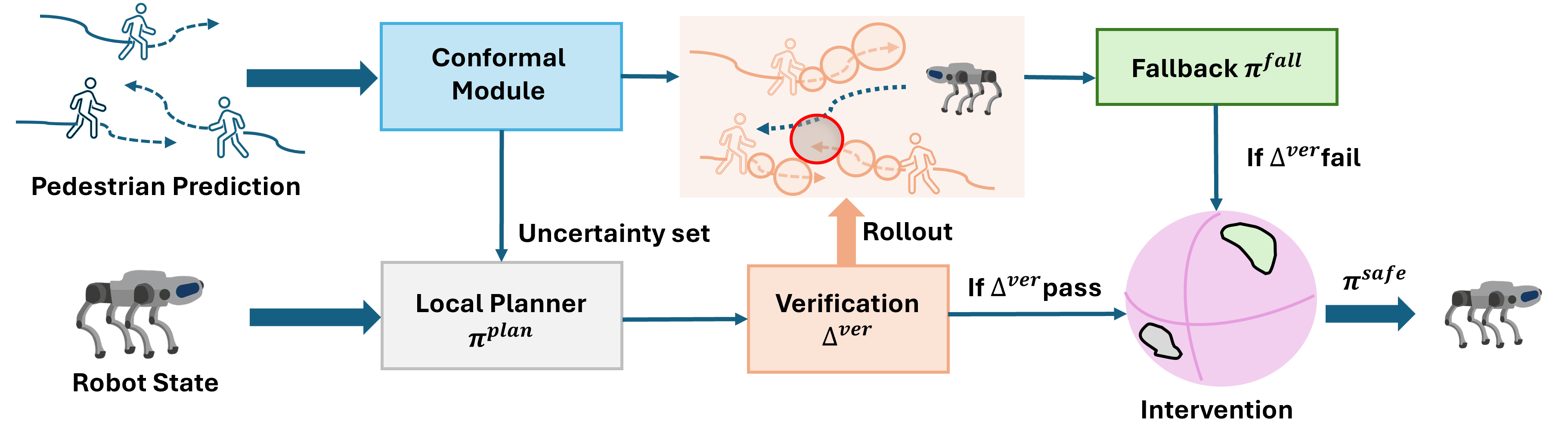}
    \caption{Overview of CoCoNav. Online conformal calibration constructs adaptive robot no-go regions for soft MPC planning. A posteriori verification accepts the nominal roll-out or invokes the certified fallback policy, after which the first control of the selected roll-out is executed}
    \label{fig:framework}
\end{figure*}
Robot navigation among pedestrians has been studied through reactive, optimization-based, and learning-based approaches \cite{mirsky2024conflict}. Classical reactive methods, including social-force models, dynamic-window search, and reciprocal velocity obstacles, construct collision-avoidance actions directly from the current configuration \cite{helbing1995social, fox2002dynamic, van2011reciprocal}. These methods are computationally attractive, whereas anticipatory navigation additionally requires reasoning about how pedestrian motion may evolve over the planning horizon.

Prediction-aware planners incorporate pedestrian forecasts into trajectory optimization or MPC \cite{brito2019model, poddar2023crowd, de2024topology, martinez2024shine}. Learning-based alternatives obtain navigation policies from demonstrations \cite{xie2021towards, pokle2019deep, karnan2022socially, chandra2024towards} or reinforcement learning \cite{chen2019crowd, everett2021collision}. In parallel, learned trajectory predictors model socially coupled and potentially multimodal pedestrian motion \cite{alahi2016social, salzmann2020trajectron, xu2022socialvae}. Regardless of the forecasting architecture, prediction errors remain consequential when a planner uses forecast positions to enforce robot--pedestrian clearance. Our work therefore focuses on online calibration of trajectory-prediction error rather than treating the output of a predictor as exact.

\subsection{Online Conformal Calibratin for Motion Planning}
\label{sec:conformal_prediction}
Uncertainty-aware planners have represented prediction uncertainty using confidence regions, Gaussian models, and stochastic or chance-constrained optimization \cite{fridovich2020confidence, fan2021step, omainska2021gaussian, nakamura2022online, nair2022stochastic}. Their validity depends on the corresponding modeling and distributional assumptions. Conformal prediction instead constructs calibrated prediction regions from nonconformity scores without requiring a parametric error distribution \cite{angelopoulos2023conformal}. Conformal regions have been incorporated into dynamic-environment planning and MPC constraints \cite{lindemann2023safe, huang2025interaction}, and have also been used to filter actions proposed by learned navigation policies \cite{strawn2023conformal}. Related decision-oriented formulations propagate conformal uncertainty into sensorimotor policy learning and downstream autonomous decisions \cite{huang2024conformal, lekeufack2024conformal}. Standard split-conformal calibration is not designed to track a changing error sequence \cite{tibshirani2019conformal}. Adaptive conformal methods update their thresholds online \cite{gibbs2021adaptive, dixit2023adaptive, yao2025towards}, while Conformal PID control interprets coverage regulation through proportional, integral, and predictive feedback components \cite{angelopoulos2023conformalpid}. We adapt its proportional--integral update to causal, horizon-indexed trajectory-error scores,
resulting in a long-run empirical miscoverage bound under its stated conditions.

\subsection{Runtime Certification for Safety Control}
Calibrating predictive uncertainty does not by itself specify how a robot should act on the resulting uncertainty sets. One option is to impose them as hard planning constraints \cite{lindemann2023safe, dixit2023adaptive, huang2025interaction}; however, large online-calibrated sets can make a constrained finite-horizon problem infeasible. Safety filters offer a complementary performance--safety separation: a task controller proposes an action, while a runtime monitor accepts or modifies it using a safety-oriented fallback \cite{hsu2023safety}. Within this taxonomy, model predictive shielding checks whether a candidate action followed by a fallback rollout can safely reach an invariant terminal set and switches to the fallback when the check fails \cite{bastani2021safe}. Predictive safety filters instead compute a minimally modified input through finite-horizon optimization with terminal safety conditions \cite{wabersich2021predictive}; probabilistic extensions replace worst-case tubes with probabilistic reachable sets \cite{wabersich2021probabilistic}. Conformal predictive safety filtering provides a related statistical interface for learned controllers in dynamic environments \cite{strawn2023conformal}. 
These methods clarify the basis of safety claims, though their guarantees depend on the dynamics, uncertainty model, terminal safe set, and fallback availability. Our scheme evaluates nominal and contingency MPC rollouts using online-calibrated error bounds, braking if neither passes. Rather than enforcing an invariant terminal set, it supports a long-run clearance-violation bound that accounts for execution, perception, and no-certified-action events.

\section{PROBLEM FORMULATION}
We consider a robot governed by
\begin{equation}
    \begin{aligned}
        x_{k+1}&=f(x_k,u_k),&
        p_k&=P(x_k)\in \mathbb{R}^2,\\
        u_k&=(v_k,\omega_k)\in \mathcal{U}.&
    \end{aligned}
    \label{eq:compact_dynamics}
\end{equation}
where $p_k$ indicates planar position and $(v_k, \omega_k)$ are the linear and angular velocities. We define a navigation task as $\mathcal{T}:=(x_{\mathrm{init}}, x_{\mathrm{goal}}, \mathcal{O}, \mathcal{S}_o, \mathcal{S}_a, \mathcal{U}, \delta)$, where $x_{\mathrm{init}}$ and $x_{\mathrm{goal}}$ are the initial and goal robot poses, $\mathcal{S}_o$ and $\mathcal{S}_a$ denote static obstacles and moving pedestrians, $\mathcal{U}$ is the robot action space, and $\delta \in (0,1)$ is the target conformal miscoverage rate. At each time step $k$, the observation $o_k=(o_k^{\mathrm{rgb}}, o_k^{\mathrm{geo}}) \in \mathcal{O}$ combines visual and geometric sensing. 

Let $\mathcal{I}_k$ be the physical pedestrians relevant to safety, with positions $y_{k,j} \in \mathbb{R}^2$, and let $\mathcal{J}_k \subseteq \mathcal{I}_k$ be those currently tracked and available to the controller. We define the signed clearance from a robot position $p$ as
\begin{equation}
    c_j(p,y):=\lVert{p-y}\rVert_2-(R_r+R_j),
    \label{eq:compact_clearance}
\end{equation}
where $R_r$ and $R_j$ are the robot and pedestrian radii. Given a required clearance $d_{\mathrm{safe}} \geq 0$, we have
\begin{equation}
    h_k:=\min_{j\in\mathcal{I}_k}c_j(p_k,y_{k,j}),\qquad
    V_k:=\mathbf{1}\{h_k<d_{\mathrm{safe}}\}.
    \label{eq:compact_violation}
\end{equation}
The objective is to reach $x_{\mathrm{goal}}$ while avoiding $\mathcal S_o\cup\mathcal S_a$ and controlling the long-run empirical frequency of $V_k$. In the ideal certified operating domain, the average nonviolation rate is at least $1-\delta$.

\section{PROPOSED METHOD}
Figure~\ref{fig:framework} illustrates the closed-loop pipeline of CoCoNav. Tracked pedestrian histories produce point-trajectory forecasts, while delayed prediction errors update horizon-wise conformal thresholds that define adaptive robot no-go regions. A soft MPC planner generates a nominal roll-out, which is accepted only after posterior verification of dynamic conformal-clearance conditions. If verification fails, the same test is applied to the fallback action library and emergency-braking roll-out; the first control of the selected roll-out is then executed before replanning.
\subsection{Pedestrian Trajectory Forecasting}
\label{sec:traj_prediction}
For each pedestrian $j \in \mathcal{J}_k$, we denote its observed trajectory over $L$ steps with
\begin{equation}
    \mathbf Y_{k,j}^{\mathrm{hist}}
    :=
    [y_{k-L+1,j},\ldots,y_{k,j}]^\top\in\mathbb{R}^{L\times2}.
    \label{eq:compact_history}
\end{equation}
We train a VAE-based predictor $\Phi_\theta$ on the ETH/UCY datasets as in \cite{xu2022socialvae} and compute $H$-steps rollout in the future
\begin{equation}
    \widehat{\mathbf Y}_{k,j}
    :=
    \Phi_\theta(\mathbf Y_{k,j}^{\mathrm{hist}})
    =
    [\widehat y_{k,j}^{1},\ldots,\widehat y_{k,j}^{H}]^\top
    \in\mathbb{R}^{H\times2}.
    \label{eq:compact_prediction}
\end{equation}
For compactness, let $Y_k:=(y_{k,j})_{j\in\mathcal{J}_k}$ and 
$\widehat Y_k^\tau:=(\widehat y_{k,j}^\tau)_{j\in\mathcal{J}_k}$ denote the joint observation and $\tau$-step forecast under a fixed track ordering, and write $\widehat Y_k^{1:H}:=(\widehat Y_k^1,\ldots,\widehat Y_k^H)$.

\subsection{Trajectory-Error Conformal Calibration}
\label{sec:conformal_pid}
\subsubsection{Casual multi-step scores}
For each horizon $\tau\in\{1,\ldots,H\}$, the $\tau$-step-ahead forecast $\widehat Y_{k-\tau}^{\tau}$, issued at time $k-\tau$ for target time $k$, is evaluated once $Y_k$ becomes available. The resulting time-lagged score updates the conformal state used to issue $q_{k+1}^{\tau}$ for the new forecast $\widehat Y_k^\tau$. Let $\mathcal{M}_k^\tau\subseteq\mathcal{J}_{k-\tau}$ contain the consistently matched tracks. For any fixed ordering of these tracks, we define
\begin{align}
    d_{k,j}^{\tau,\mathrm{pred}}
    &:=
    \lVert{y_{k,j}-\widehat y_{k-\tau,j}^{\tau}}\rVert_2,
    \label{eq:compact_individual_error}\\
    \mathbf z_k^\tau
    &:=
    \operatorname{col}_{j\in\mathcal M_k^\tau}
    (y_{k,j}-\widehat y_{k-\tau,j}^{\tau}),\nonumber\\
    s_k^\tau
    &:=
    \lVert{\mathbf z_k^\tau}\rVert_2
    =
    \left(
    \sum_{j\in\mathcal M_k^\tau}
    (d_{k,j}^{\tau,\mathrm{pred}})^2
    \right)^{1/2}.
    \label{eq:compact_joint_score}
\end{align}
This stacked Euclidean score is premutation invariant and directly satisfies the score-dominance relation
\begin{equation}
    d_{k,j}^{\tau,\mathrm{pred}}\leq s_k^\tau,
    \qquad j\in\mathcal M_k^\tau.
    \label{eq:compact_score_dominance}
\end{equation}
Each forecast, track set, and associated threshold is stored when issued. Missing or unmatched safety-relevant pedestrians are charged separately to the perception exception term. 

\subsubsection{Conformal quantile update}
For every $\tau$, the threshold $q_{k+1}^{\tau,\mathrm{raw}}$ is issued at planning step $k$ together with $\widehat Y_k^\tau$. Therefore, the delayed score $s_k^\tau$, which evaluates the forecast issued at $k-\tau$, is compared with the stored threshold 
$q_{k-\tau+1}^{\tau,\mathrm{raw}}$. Once the score is observed, set
\begin{equation}
\begin{aligned}
    e_k^{\tau,\mathrm{raw}}
    &:=
    \mathbf{1}\{s_k^\tau>q_{k-\tau+1}^{\tau,\mathrm{raw}}\},&
    g_k^\tau&:=e_k^{\tau,\mathrm{raw}}-\delta,\\
    G_k^\tau
    &:=
    \begin{cases}
        \sum_{i=\tau}^{k}g_i^\tau,&k\geq\tau,\\
        0,&k<\tau.
    \end{cases}
\end{aligned}
    \label{eq:compact_error}
\end{equation}
Each horizon maintains an independent conformal state
\begin{align}
    P_{k+1}^\tau
    &=P_k^\tau+\eta_k^\tau g_k^\tau,
    \label{eq:compact_p_update}\\
    q_{k+1}^{\tau,\mathrm{raw}}
    &=P_{k+1}^\tau+r_k(G_{k-1}^\tau),
    \label{eq:compact_q_update}\\
    r_k(x)
    &:=
    K_I\widetilde{\tan}\!\left(
    \frac{x\log(k+1)}{C_{\mathrm{sat}}(k+1)}
    \right),
    \label{eq:compact_integrator}
\end{align}
where $\widetilde{\tan}(z)$ equals $\tan(z)$ on $(-\pi/2,\pi/2)$ and saturates to the corresponding signed infinity outside this interval. The optional proportional scaling sets 
$\eta_k^\tau$ to the base learning rate times the recent finite score range. Since the raw recurrence may be negative, verification uses the physical budget
\begin{equation}
    q_k^{\tau,\mathrm{eff}}
    :=
    \max\{0,q_k^{\tau,\mathrm{raw}}\}.
    \label{eq:compact_effective_q}
\end{equation}
Consequently,
$\mathbf{1}\{s_k^\tau>q_{k-\tau+1}^{\tau,\mathrm{eff}}\}
\leq e_k^{\tau,\mathrm{raw}}$.
The region issued at time $k$ for horizon $\tau$ is the joint norm ball
\begin{equation}
    C_k^\tau
    :=
    \left\{
    (y_j)_{j\in\mathcal{J}_k}:
    \left(
    \sum_{j\in\mathcal{J}_k}
    \lVert{y_j-\widehat y_{k,j}^{\tau}}\rVert_2^2
    \right)^{1/2}
    \leq q_{k+1}^{\tau,\mathrm{eff}}
    \right\}.
    \label{eq:compact_prediction_set}
\end{equation}
To obtain the final clearance-violation guarantee, we can prove the following proposition that bounds the empirical one-step miscoverage frequency.
\begin{proposition}[One-sided empirical calibration]
\label{prop:compact_cpi}
Assume $K_I,C_{\mathrm{sat}}>0$, all scores and learning rates are finite.
For
\[
    a_t:=\frac{\pi C_{\mathrm{sat}}}{2}
    \frac{t+1}{\log(t+1)},\qquad
    \bar a_t:=\max_{1\leq i\leq t}a_i,
\]
the one-step stream ($\tau=1$) satisfies, for $T\geq2$,
\begin{equation}
    \frac1T\sum_{t=1}^{T}e_t^{1,\mathrm{raw}}
    \leq
    \delta+r_T,\qquad
    r_T:=\frac{\bar a_{T-1}+3}{T}\longrightarrow0.
    \label{eq:compact_cpi_bound}
\end{equation}
\end{proposition}
\begin{proof}
For $\tau=1$, saturation makes $q_{t+1}^{1,\mathrm{raw}}=+\infty$ whenever
$G_{t-1}^1\geq a_t$, so the next finite score is covered. The one-round delay permits an overshoot of at most two, and induction gives $G_t^1\leq\bar a_{t-1}+2$. Accounting for initialization adds at most one; using $G_T^1=\sum_{t=1}^{T}(e_t^{1,\mathrm{raw}}-\delta)$ yields \eqref{eq:compact_cpi_bound}. Finally, $\bar a_T=O(T/\log T)$ implies $r_T\to0$.
\end{proof}
\subsection{Uncertainty-Aware MPC}
\label{sec:mpc_controller}
\subsubsection{Nominal candidate}
At time $k$, the local planner computes a nominal sequence $U_k=(u_{k|k},\ldots,u_{k+H-1|k})$ and state roll-out $X_k=(x_{k|k},\ldots,x_{k+H|k})$ by solving
\begin{subequations}
\label{eq:compact_mpc}
\begin{align}
    \min_{X_k,U_k}\quad&
    J_{\mathrm{track}}(X_k,U_k)
    +
    \sum_{\tau=1}^{H}\sum_{j\in\mathcal{J}_k}
    \psi\!\left(d_{k,j}^{\tau}\right)
    \label{eq:compact_mpc_obj}\\
    \mathrm{s.t.}\quad&
    x_{k|k}=x_k,
    \nonumber\\
    &
    x_{k+\tau+1|k}
    =
    f(x_{k+\tau|k},u_{k+\tau|k}),
    \quad \tau=0,\ldots,H-1,
    \nonumber\\
    &
    u_{k+\tau|k}\in\mathcal{U},
    \quad \tau=0,\ldots,H-1,
    \label{eq:compact_mpc_dyn}\\
    d_{k,j}^{\tau}
    &:=
    c_j(p_{k+\tau|k},\widehat y_{k,j}^{\tau})
    -q_{k+1}^{\tau,\mathrm{eff}}.
    \label{eq:compact_soft_margin}
\end{align}
\end{subequations}
Here $J_{\mathrm{track}}$ collects path tracking and velocity control terms, $\psi$ is a decreasing smooth penalty. 
\subsubsection{A posteriori certification}
Let $Z\in\mathcal Z_k$ denote a dynamically feasible horizon-$H$ robot roll-out from $x_k$, with predicted position $p_{k+\tau|k}^Z$ at horizon $\tau$. Any nominal or fallback roll-out $Z$ is certified only if
\begin{equation}
    c_j(p_{k+\tau|k}^{Z},\widehat y_{k,j}^{\tau})
    \geq
    d_{\mathrm{safe}}+
    q_{k+1}^{\tau,\mathrm{eff}}+
    \varepsilon,
    \quad
    \substack{\tau=1,\ldots,H,\\j\in\mathcal{J}_k},
    \label{eq:compact_certificate}
\end{equation}
where $\varepsilon\geq0$ is a verification reserve. We then define the worst-case certified margin and validity indicator
\begin{align}
    \mu_{k,j}^\tau(Z)
    &:=
    c_j(p_{k+\tau|k}^{Z},\widehat y_{k,j}^{\tau})
    -q_{k+1}^{\tau,\mathrm{eff}}-d_{\mathrm{safe}}-\varepsilon,
    \label{eq:compact_cert_residual}\\
    m_k(Z)
    &:=
    \min_{\substack{1\leq\tau\leq H\\j\in\mathcal{J}_k}}
    \mu_{k,j}^\tau(Z),
    \label{eq:compact_robust_margin}\\
    \Gamma_k(Z)
    &:=
    \mathbf{1}\{m_k(Z)\geq0\}.
    \label{eq:compact_cert_indicator}
\end{align}
Let $N_k$ indicate that neither the nominal, nor the fallback action can be certified, and \(B_{k+1}\) for any missing pedestrian tracks.
Algorithm~\ref{alg:compact_method} summarizes the receding-horizon loop, including conformal updates for trajectory prediction, soft MPC and posterior certification.
The full-horizon test is operationally conservative, and the theorem below certifies only the first control that is actually executed. 
\begin{algorithm}[t]
\caption{Trajectory-Error CPI--MPC Pipeline}
\label{alg:compact_method}
\begin{algorithmic}[1]
\Require $\delta,H,d_{\mathrm{safe}},\varepsilon$ and CPI/fallback states
\State Observe $(x_k,Y_k)$ and retrieve the forecast archive
\State Evaluate available delayed scores $s_k^\tau$ by
\eqref{eq:compact_joint_score}
\State Update $\{P_{k+1}^\tau,G_k^\tau,
q_{k+1}^{\tau,\mathrm{eff}}\}_{\tau=1}^{H}$ by
\eqref{eq:compact_error}--\eqref{eq:compact_effective_q}
\State Predict $\widehat Y_k^{1:H}$ using $\Phi_\theta$
\State Solve the soft MPC \eqref{eq:compact_mpc} for $(X_k^*,U_k^*)$
\State $Z_k^{\mathrm{nom}}\gets
\operatorname{Rollout}(x_k,U_k^*)$
\If{$\Gamma_k(Z_k^{\mathrm{nom}})=1$}
    \State $Z_k^{\mathrm{sel}}\gets Z_k^{\mathrm{nom}}$
\Else
    \State Verify the library:
    $\mathcal C_k^{\mathrm{lib}}\gets
    \{Z\in\mathcal Z_k^{\mathrm{lib}}:\Gamma_k(Z)=1\}$
    \If{$\mathcal C_k^{\mathrm{lib}}\neq\varnothing$}
        \State $Z_k^{\mathrm{sel}}\gets
        \operatorname*{arg\,max}_{Z\in\mathcal C_k^{\mathrm{lib}}}
        [m_k(Z)-J_{\mathrm{sec}}(Z)]$
    \Else
        \State $Z_k^{\mathrm{sel}}\gets Z_k^{\mathrm{br}}$
    \EndIf
\EndIf
\State $N_k\gets\mathbf{1}\{\Gamma_k(Z_k^{\mathrm{sel}})=0\}$
\State Archive the issued forecasts, identities, and thresholds
\State Execute $u_k\gets\Pi_0(Z_k^{\mathrm{sel}})$ and replan
\end{algorithmic}
\end{algorithm}

\begin{theorem}[Fallback-aware clearance bound]
\label{thm:compact_main}
Under Proposition~\ref{prop:compact_cpi}, causal forecast logging, and
\eqref{eq:compact_score_dominance}, Algorithm~\ref{alg:compact_method}
satisfies
\begin{equation}
    \frac1T\sum_{k=0}^{T-1}V_{k+1}
    \leq
    \delta+r_T
    +\frac1T\sum_{k=0}^{T-1}
    \left(B_{k+1}+N_k\right).
    \label{eq:compact_main_bound}
\end{equation}
\end{theorem}

\begin{proof}
On a certified round with
$s_{k+1}^1\leq q_{k+1}^{1,\mathrm{eff}}$ and $B_{k+1}=0$,
\eqref{eq:compact_score_dominance} gives
$\lVert{y_{k+1,j}-\widehat y_{k,j}^{1}}\rVert_2
\leq q_{k+1}^{1,\mathrm{eff}}$ for every safety-relevant pedestrian.
The one-step case of \eqref{eq:compact_certificate} and the reverse triangle inequality give
\[
    c_j(p_{k+1},y_{k+1,j})
    \geq
    c_j(p_{k+1|k}^{\mathrm{sel}},\widehat y_{k,j}^{1})
    -q_{k+1}^{1,\mathrm{eff}}
    \geq d_{\mathrm{safe}}+\varepsilon.
\]
Hence, we have pathwise \(V_{k+1}\leq e_{k+1}^{1,\mathrm{raw}}+B_{k+1}+N_k\).
Summation and Proposition~\ref{prop:compact_cpi} prove \eqref{eq:compact_main_bound}.
\end{proof}

\subsection{Certified Contingency Control}
If the nominal roll-out fails certification, the controller verifies a finite library of hold, slow, and straight/turning brake roll-outs:
\begin{equation}
    \mathcal C_k^{\mathrm{lib}}
    =
    \{Z\in\mathcal Z_k^{\mathrm{lib}}:\Gamma_k(Z)=1\}.
    \label{eq:compact_cert_library_short}
\end{equation}
The executed roll-out is selected by
\begin{equation}
Z_k^{\mathrm{sel}}\in
\begin{cases}
\displaystyle
\arg\max_{Z\in\mathcal C_k^{\mathrm{lib}}}
[m_k(Z)-J_{\mathrm{sec}}(Z)],
&\mathcal C_k^{\mathrm{lib}}\neq\varnothing,\\[1mm]
\{Z_k^{\mathrm{br}}\},
&\mathcal C_k^{\mathrm{lib}}=\varnothing,
\end{cases}
\label{eq:compact_fallback_short}
\end{equation}
where $J_{\mathrm{sec}}$ ranks already-certified candidates by speed and smoothness, and $Z_k^{\mathrm{br}}$ is an acceleration-limited stopping roll-out. Braking is certified only if
$\Gamma_k(Z_k^{\mathrm{br}})=1$; else $N_k=\mathbf{1}\{\Gamma_k(Z_k^{\mathrm{sel}})=0\}$ records that no certified action exists. 
\section{EXPERIMENTS AND RESULTS}
\label{sec:experiments}
\subsection{Simulation Environment}
We evaluate the proposed framework in \texttt{MATRiX}\footnote{\url{https://github.com/zsibot/matrix}}, a high-fidelity simulation platform that integrates MuJoCo, Unreal Engine 5, and CARLA to provide interactive environments for robotics research. We design three crowd-navigation scenarios within a 10m\(\times\)10m area, each populated with 4 to 12 pedestrians. In every scenario, the robot must traverse the area from bottom to top while avoiding both static obstacles and moving pedestrians. The three scenarios differ in the structure of the crowd flow: \textit{Opposing Flow}, in which the robot navigates against a crowd moving in the opposite direction; \textit{Intersecting Flow}, in which it crosses a stream of pedestrians moving perpendicular to its path; and \textit{Unstructured Flow}, in which it navigates through a crowd with random, multi-directional motion. Figure \ref{fig:exp_env} illustrates an example of the real and simulated environments.

\begin{figure}[t]
    \centering
    \begin{minipage}[b]{0.48\columnwidth}
        \begin{subfigure}[b]{\linewidth}
            \centering
            \includegraphics[width=\linewidth]{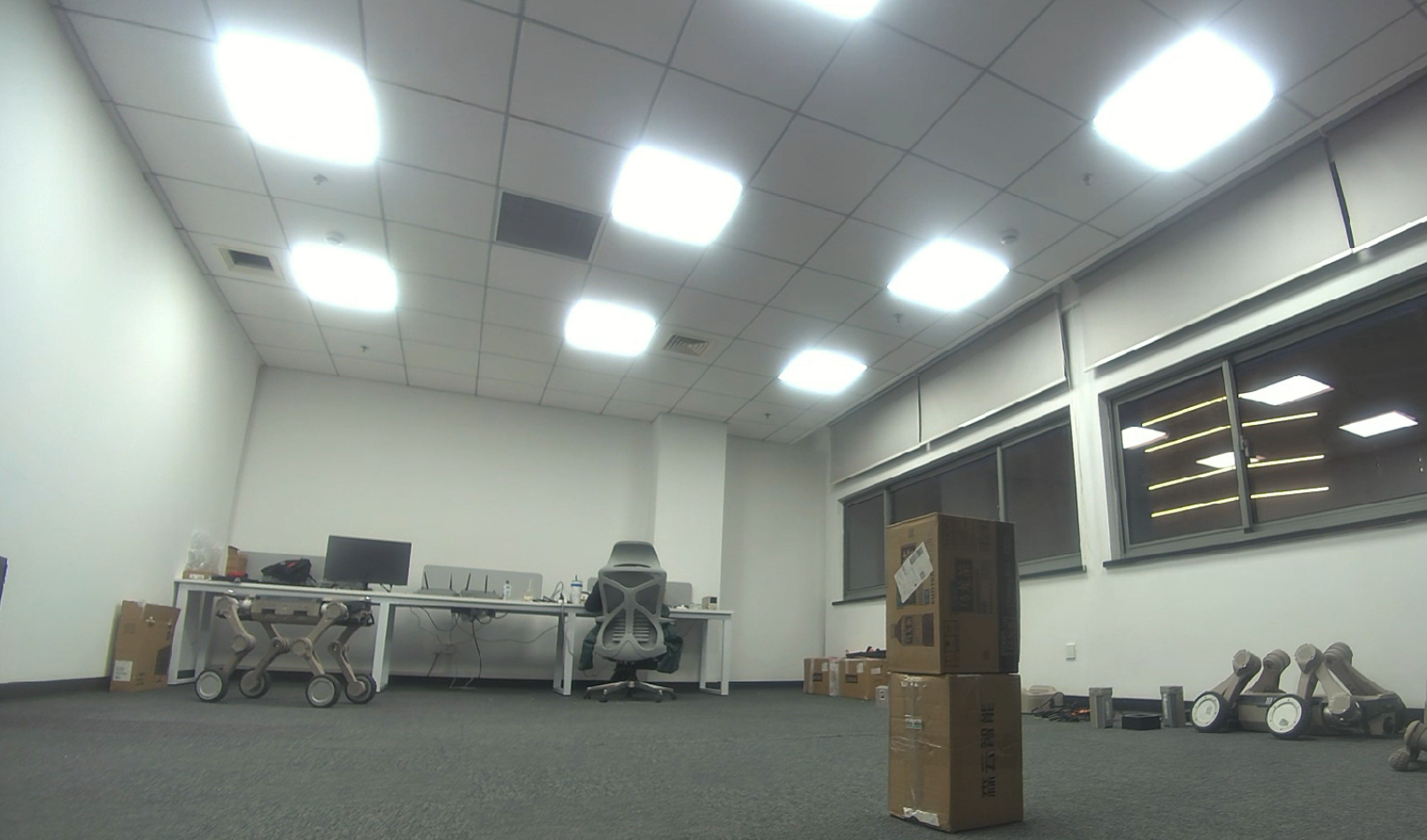}
            \caption{The real-world environment.}
            \label{fig:env_real}
        \end{subfigure}
        
        \vspace{0.3cm} 
        
        \begin{subfigure}[b]{\linewidth}
            \centering
            \includegraphics[width=\linewidth]{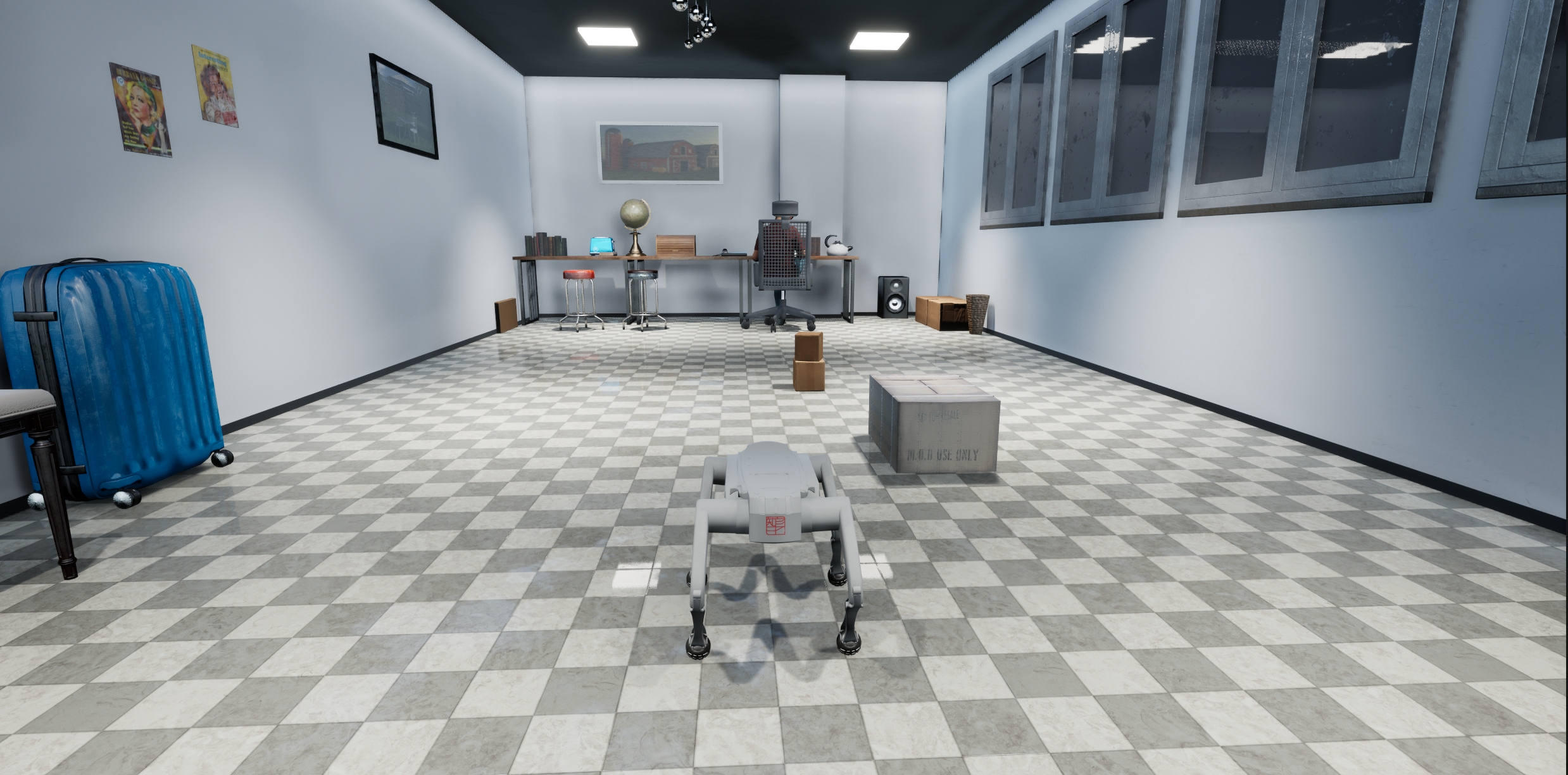}
            \caption{The simulated environment.}
            \label{fig:env_sim}
        \end{subfigure}
    \end{minipage}
    \hfill 
    \begin{minipage}[b]{0.48\columnwidth}
        \centering
        \begin{subfigure}[b]{\linewidth}
            \centering
            \includegraphics[width=\linewidth, height=2.25in]{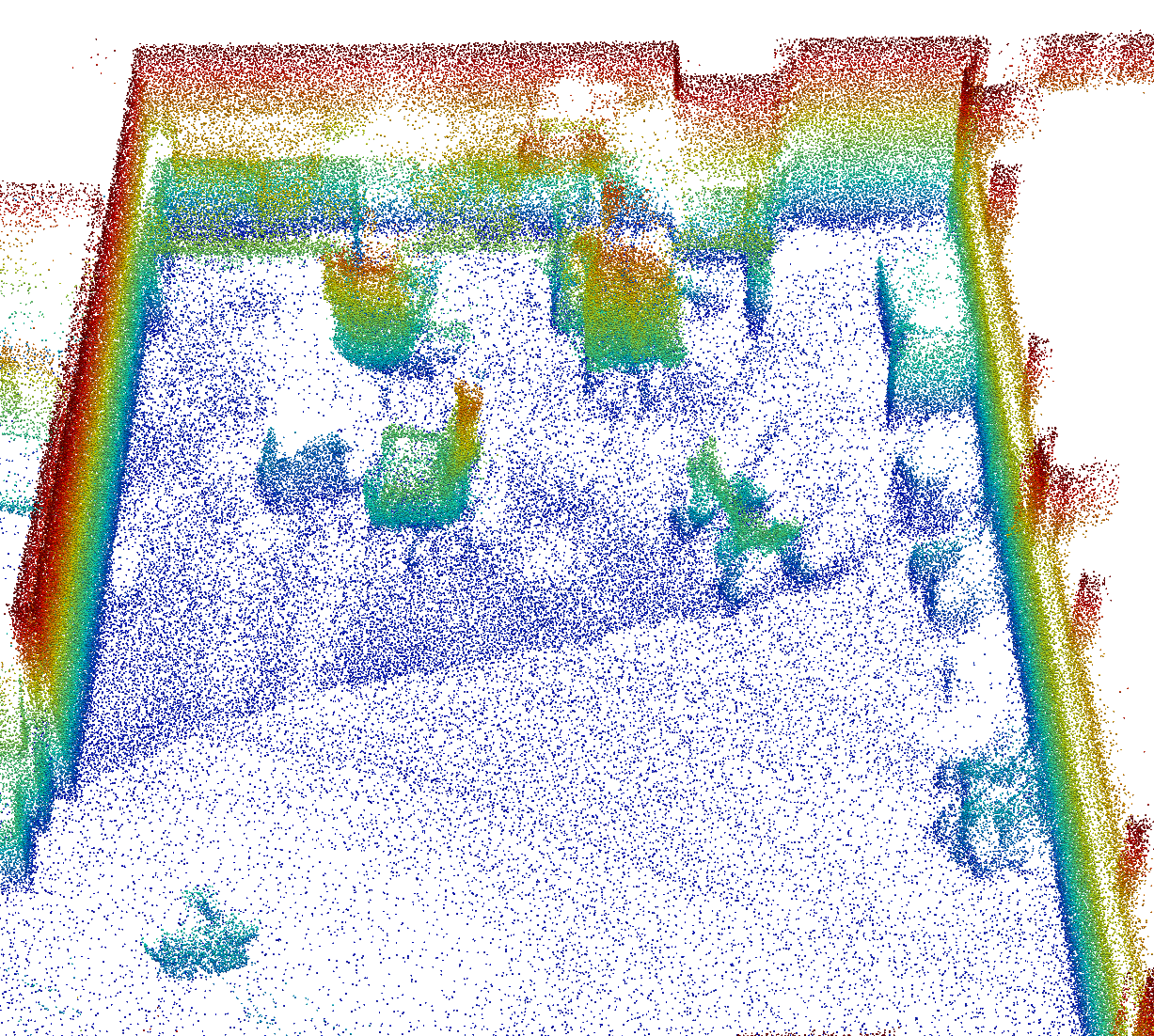} 
            \caption{The reconstructed environment with pointcloud.}
            \label{fig:env_cloud}
        \end{subfigure}
    \end{minipage}
    
    \caption{Snapshots of the environment at different levels.}
    \label{fig:exp_env}
\end{figure}

\subsection{Baselines}
We select three baseline methods for comparison. Since our framework targets a real-world robotic pipeline, each baseline must be adapted to integrate with our perception and control stack.
We therefore prioritize methods with open-source implementations that are compatible with our hierarchical navigation framework. The selected baselines are: \textbf{ST-Planner} \cite{han2023efficient}, a spatiotemporal global planner that supports real-time replanning, with the resulting path tracked by a Pure Pursuit controller; \textbf{ACP-MPC} \cite{dixit2023adaptive}, an MPC-based local planner that incorporates adaptive conformal prediction; and \textbf{MPPI} \cite{williams2016aggressive}, a sampling-based model predictive controller that computes commands from a cost-weighted average over multiple parallel trajectory rollouts.

\subsection{Evaluation Metrics}
We quantify the planners' performance in crowd navigation using the following metrics, as in \cite{mirsky2024conflict}. \textbf{Success Rate} (SR) measures the percentage of episodes in which the robot reaches its goal within the specified time limit, i.e., \(SR=N_\text{success}/N_\text{total}\). \textbf{Minimum Distance} (MD) measures the average minimum distance between the robot and pedestrians during the navigation task, i.e., \(MD=(1/N_\text{total})\sum_{k=1}^{N_\text{total}}D_k\). \textbf{Collision Rate} (CR) measures the percentage of episodes in which at least one collision (MD below the threshold) occurs, i.e., \(CR=N_\text{collision}/N_\text{total}\). \textbf{Navigation Time} (NT) measures the average travel time to the goal across all test episodes, i.e., \(NT=(1/N_\text{total})\sum_{k=1}^{N_\text{total}}T_k\). \textbf{Path Length} (PL) measures the average traveled distance to the goal across all test episodes, i.e., \(PL=(1/N_\text{total})\sum_{k=1}^{N_\text{total}}L_k\).
These metrics collectively reflect the effectiveness (SR), safety (MD, CR), and efficiency (NT, PL) of the planners.

\subsection{Implementation Details}
Our implementation of CoCoNav is written in C++/ROS2 and will be open sourced upon acceptance. Simulations are conducted on an Ubuntu 22.04 system equipped with an NVIDIA RTX 4090 GPU. 
For trajectory prediction, we employ a conditional variational autoencoder that models the future pedestrian motion. The network encodes each agent's motion history with a recurrent backbone and aggregates neighboring agents through an attention-based social pooling module, capturing interactions among pedestrians. Conditioned on these representations, the decoder samples plausible future trajectories from the learned latent space. We pair the predictor with downsampling and adopt an 8-step history window and a 12-step prediction horizon.
Additionally, we port CPI to C++ and enable adaptive parameter tuning.
We solve local planning via Ceres 
Pedestrian and robot radii are set to 0.25m and 0.35m, respectively.

\begin{table*}[th]
\centering
\caption{Quantitative comparison between our method and baselines across various simulated crowd scenarios.}
\label{tab:quantitative_comparison}
\setlength{\tabcolsep}{1.5pt} 
\renewcommand{\arraystretch}{1.0} 

\resizebox{\textwidth}{!}{%
\begin{tabular}{@{}ll ccccc ccccc ccccc@{}} 
\toprule
\multirow{2}{*}{\textbf{Env}} & \multirow{2}{*}{\textbf{Method}} & 
\multicolumn{5}{c}{\textbf{Ped=4}} & \multicolumn{5}{c}{\textbf{Ped=8}} & \multicolumn{5}{c}{\textbf{Ped=12}} \\ 
\cmidrule(lr){3-7} \cmidrule(lr){8-12} \cmidrule(l){13-17} 

 &  & SR $\uparrow$ & \shortstack{MD(m)} $\uparrow$ & CR $\downarrow$ & \shortstack{NT(s)} $\downarrow$ & \shortstack{PL(m)} $\downarrow$
    & SR $\uparrow$ & \shortstack{MD(m)} $\uparrow$ & CR $\downarrow$ & \shortstack{NT(s)} $\downarrow$ & \shortstack{PL(m)} $\downarrow$
    & SR $\uparrow$ & \shortstack{MD(m)} $\uparrow$ & CR $\downarrow$ & \shortstack{NT(s)} $\downarrow$ & \shortstack{PL(m)} $\downarrow$ \\ \midrule

\multirow{4}{*}{\textbf{\shortstack[l]{Opposing\\Flow}}} 
 & ST-Planner & $100\%$ & $0.46{\pm}0.26$ & $20\%$ & $34.45{\pm}3.80$ & $14.95{\pm}0.34$ & $100\%$ & $0.42{\pm}0.02$ & $100\%$ & $37.87{\pm}0.94$ & $15.04{\pm}0.57$ & $100\%$ & $0.39{\pm}0.19$ & $80\%$ & $32.75{\pm}1.36$ & $14.32{\pm}0.33$ \\
 & MPPI       & $100\%$ & $0.62{\pm}0.09$ & $10\%$ & $32.03{\pm}4.78$ & $14.89{\pm}0.68$ & $80\%$ & $0.59{\pm}0.12$ & $10\%$ & $36.82{\pm}5.46$ & $15.01{\pm}0.72$ & $80\%$ & $0.53{\pm}0.12$ & $20\%$ & $39.77{\pm}5.67$ & $16.08{\pm}1.29$ \\
 & ACP-MPC    & $80\%$ & $1.02{\pm}0.21$ & $0\%$ & $40.73{\pm}3.61$ & $17.53{\pm}0.43$ & $0\%$ & $0.61{\pm}0.02$ & $0\%$ & $102.49{\pm}4.34$ & $33.74{\pm}2.08$ & $0\%$ & $0.62{\pm}0.04$ & $0\%$ & $97.86{\pm}2.12$ & $32.94{\pm}1.88$ \\ \cmidrule(l){2-17} 
 \rowcolor{shadowColor} \cellcolor{white}
 & Ours       & $100\%$ & $0.66{\pm}0.10$ & $0\%$ & $36.75{\pm}2.28$ & $16.57{\pm}0.69$ & $90\%$ & $0.64{\pm}0.07$ & $0\%$ & $39.43{\pm}2.87$ & $15.36{\pm}0.42 $ & $100\%$ & $0.62{\pm}0.07$ & $0\%$ & $31.00{\pm}4.26$ & $14.83{\pm}0.83$ \\ \midrule

\multirow{4}{*}{\textbf{\shortstack[l]{Intersecting\\Flow}}}
 & ST-Planner & $100\%$ & $0.61{\pm}0.08$ & $0\%$ & $33.74{\pm}1.75$ & $14.74{\pm}0.38$ & $30\%$ & $0.42{\pm}0.01$ & $100\%$ & $43.76{\pm}1.98$ & $15.83{\pm}0.50$ & $10\%$ & $0.34{\pm}0.05$ & $100\%$ & $45.15{\pm}3.68$ & $15.63{\pm}0.23$ \\
 & MPPI       & $80\%$ & $0.62{\pm}0.02$ & $0\%$ & $37.99{\pm}2.73$ & $15.70{\pm}0.67$ & $10\%$ & $0.56{\pm}0.09$ & $10\%$ & $48.08{\pm}3.61$ & $16.23{\pm}0.34$ & $10\%$ & $0.40{\pm}0.05$ & $100\%$ & $43.73{\pm}1.70$ & $15.30{\pm}0.25$ \\
 & ACP-MPC    & $0\%$ & $0.66{\pm}0.16$ & $10\%$ & $51.34{\pm}6.59$ & $18.89{\pm}1.62$ & $0\%$ & $0.54{\pm}0.09$ & $10\%$ & $105.81{\pm}7.80$ & $33.06{\pm}3.85$ & $0\%$ & $0.50{\pm}0.12$ & $20\%$ & $117.65{\pm}7.23$ & $41.25{\pm}3.24$ \\ \cmidrule(l){2-17} 
  \rowcolor{shadowColor} \cellcolor{white}
 & Ours       & $100\%$ & $0.67{\pm}0.07$ & $0\%$ & $31.32{\pm}0.87$ & $14.78{\pm}0.11$ & $80\%$ & $0.55{\pm}0.02$ & $0\%$ & $42.75{\pm}10.12$ & $15.62{\pm}1.23$ & $80\%$ & $0.53{\pm}0.05$ & $10\%$ & $42.44{\pm}3.21$ & $15.36{\pm}0.30$ \\ \midrule

\multirow{4}{*}{\textbf{\shortstack[l]{Unstructured\\Flow}}}
 & ST-Planner & $100\%$ & $0.54{\pm}0.06$ & $20\%$ & $32.04{\pm}2.55$ & $14.27{\pm}1.07$ & $100\%$ & $0.64{\pm}0.06$ & $0\%$ & $33.47{\pm}3.03$ & $14.74{\pm}0.41$ & $30\%$ & $0.52{\pm}0.17$ & $20\%$ & $40.43{\pm}1.81$ & $15.17{\pm}0.36$ \\
 & MPPI       & $80\%$ & $0.59{\pm}0.06$ & $10\%$ & $35.38{\pm}3.54$ & $15.28{\pm}0.62$ & $20\%$ & $0.52{\pm}0.11$ & $20\%$ & $48.74{\pm}4.66$ & $15.94{\pm}0.66$ & $10\%$ & $0.54{\pm}0.09$ & $80\%$ & $42.12{\pm}2.46$ & $15.28{\pm}0.29$ \\
 & ACP-MPC    & $80\%$ & $0.82{\pm}0.28$ & $0\%$ & $46.09{\pm}11.68$ & $19.03{\pm}3.15$ & $30\%$ & $0.49{\pm}0.31$ & $20\%$ & $58.40{\pm}16.08$ & $20.61{\pm}2.87$ & $0\%$ & $0.84{\pm}0.14$ & $0\%$ & $81.99{\pm}8.20$ & $23.90{\pm}4.25$ \\ \cmidrule(l){2-17} 
  \rowcolor{shadowColor} \cellcolor{white}
 & Ours       & $100\%$ & $0.79{\pm}0.13$ & $0\%$ & $27.85{\pm}0.69$ & $14.45{\pm}0.04$ & $100\%$ & $0.52{\pm}0.06$ & $10\%$ & $34.09{\pm}1.87$ & $14.86{\pm}0.16$ & $80\%$ & $0.58{\pm}0.08$ & $0\%$ & $37.12{\pm}5.74$ & $15.76{\pm}0.67$ \\ \bottomrule
\end{tabular}%
}
\end{table*}

\begin{figure*}[t]
    \centering
    \setlength{\tabcolsep}{1pt} 
    \begin{minipage}[c]{0.03\linewidth}
        \centering\rotatebox{90}{\scriptsize \textbf{Ours}}
    \end{minipage}%
    \hfill
    \begin{minipage}[c]{0.95\linewidth}
        \centering
        \includegraphics[width=0.23\linewidth]{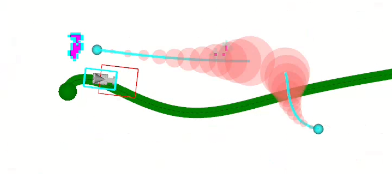}\hfill
        \includegraphics[width=0.23\linewidth]{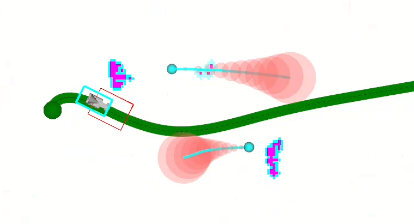}\hfill
        \includegraphics[width=0.23\linewidth]{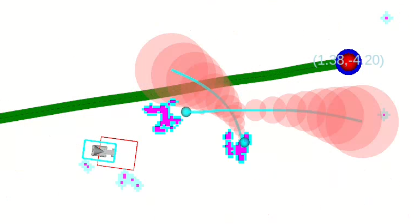}\hfill
        \includegraphics[width=0.23\linewidth]{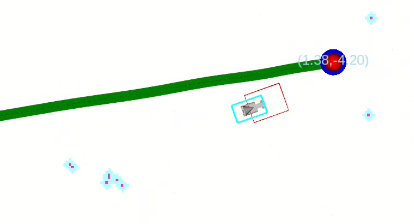}
    \end{minipage}
    \vspace{-4pt}
    \begin{minipage}[c]{0.03\linewidth}
        \centering\rotatebox{90}{\scriptsize \textbf{ST-Planner}}
    \end{minipage}%
    \hfill
    \begin{minipage}[c]{0.95\linewidth}
        \centering
        \includegraphics[width=0.23\linewidth]{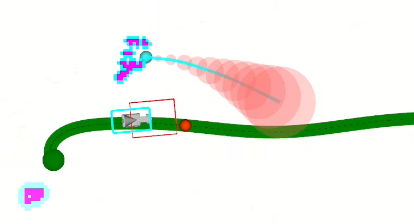}\hfill
        \includegraphics[width=0.23\linewidth]{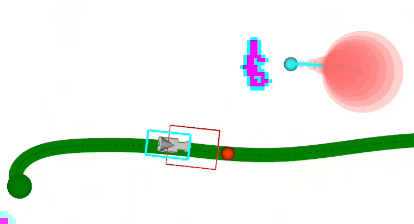}\hfill
        \includegraphics[width=0.23\linewidth]{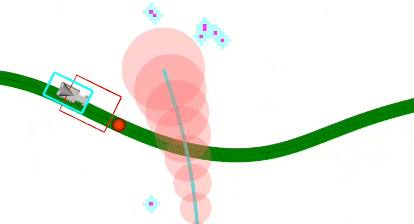}\hfill
        \includegraphics[width=0.23\linewidth]{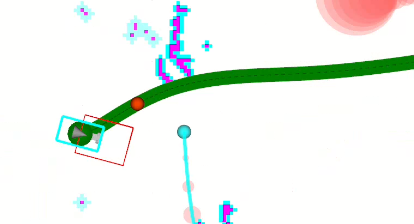}
    \end{minipage}
    \vspace{-4pt}
    \begin{minipage}[c]{0.03\linewidth}
        \centering\rotatebox{90}{\scriptsize \textbf{ACP-MPC}}
    \end{minipage}%
    \hfill
    \begin{minipage}[c]{0.95\linewidth}
        \centering
        \includegraphics[width=0.23\linewidth]{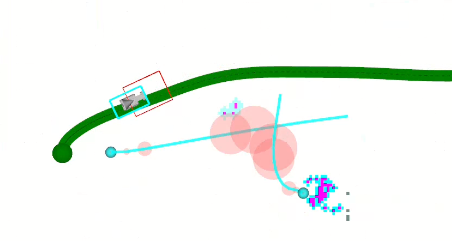}\hfill
        \includegraphics[width=0.23\linewidth]{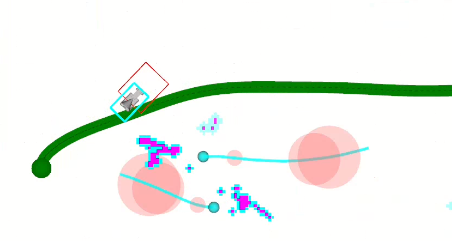}\hfill
        \includegraphics[width=0.23\linewidth]{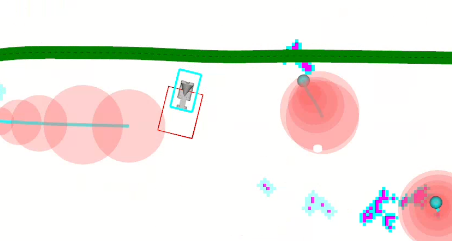}\hfill
        \includegraphics[width=0.23\linewidth]{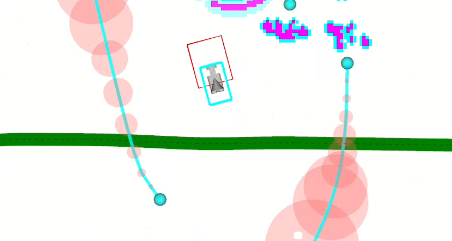}
    \end{minipage}
    \vspace{-2pt} 
    \begin{minipage}[c]{0.03\linewidth}
        \hfill 
    \end{minipage}%
    \hfill
    \begin{minipage}[c]{0.95\linewidth}
        \centering \footnotesize 
        \makebox[0.24\linewidth]{$t=t_1$}\hfill
        \makebox[0.24\linewidth]{$t=t_2$}\hfill
        \makebox[0.24\linewidth]{$t=t_3$}\hfill
        \makebox[0.24\linewidth]{$t=t_4$}
    \end{minipage}
    \caption{Sequential snapshots of robot behaviors produced by our approach and baselines when navigating through crowds. The timestamps at the bottom indicate different physical time for each navigation episode.}
    \label{fig:qualitative_results}
\end{figure*}

\subsection{Comparison Experiments}
\label{sec:comparison_exps}
\subsubsection{Quantitative analysis}

Table \ref{tab:quantitative_comparison} compares our method against the baselines across three scenarios and three crowd densities, with each method run for 10 trials per configuration and the start and goal positions sampled from specified regions. In sparse crowds (Ped=4), most planners attain a high SR with few collisions, yet our method already stands out, reaching a $100\%$ SR at a $0\%$ CR in all scenarios while keeping a large MD at a short NT and PL. As the density grows to 8 and 12 pedestrians, the baselines degrade sharply, whereas our method sustains an SR of at least $80\%$ and a CR of at most $10\%$ throughout; in Intersecting Flow at Ped=12, for instance, both ST-Planner and MPPI collapse to a mere $10\%$ SR while our method still reaches $80\%$.
Notably, the baselines fail for opposite reasons: ST-Planner relies on reactive replanning and thus incurs a high CR under dense flows (e.g., $100\%$ CR in Opposing and Intersecting Flow at Ped=8), whereas ACP-MPC keeps CR low but grows overly conservative, inflating NT beyond twice ours (e.g., $102.49$s vs.\ $39.43$s in Opposing Flow at Ped=8) so that most episodes time out and its SR collapses to $0\%$. Our method avoids both failure modes, combining efficient navigation with a low collision rate.

\subsubsection{Qualitative evaluation}
Figure \ref{fig:qualitative_results} presents sequential snapshots of the Unstructured Flow scenario (Ped=8), where the green curve traces the robot trajectory and the red regions denote the predicted pedestrian occupancy, while Fig. \ref{fig:ours_robot_data} details the corresponding motion profiles for our method and ACP-MPC. As shown in the first row of Fig. \ref{fig:qualitative_results}, our method anticipates the crowd and initiates early, smooth avoidance, staying close to the global plan and reaching the goal by $t_4$. This is corroborated by the top row of Fig. \ref{fig:ours_robot_data}, where the executed trajectory closely follows the global plan and the predominantly forward linear velocity together with bounded angular commands ($|\omega|\!\le\!0.6$\,rad/s) drive the robot to the goal in about 42\,s. In contrast, ST-Planner reacts only to instantaneous observations, so pedestrians intrude directly onto its tracked path and induce a near-collision. ACP-MPC instead produces excessively large conformal sets that force overly conservative maneuvers; accordingly, the bottom row of Fig. \ref{fig:ours_robot_data} exhibits a pronounced detour loop and highly oscillatory velocity commands, requiring over 100\,s to complete the episode.

\begin{figure*}[t]
    \centering
    \includegraphics[width=0.30\linewidth]{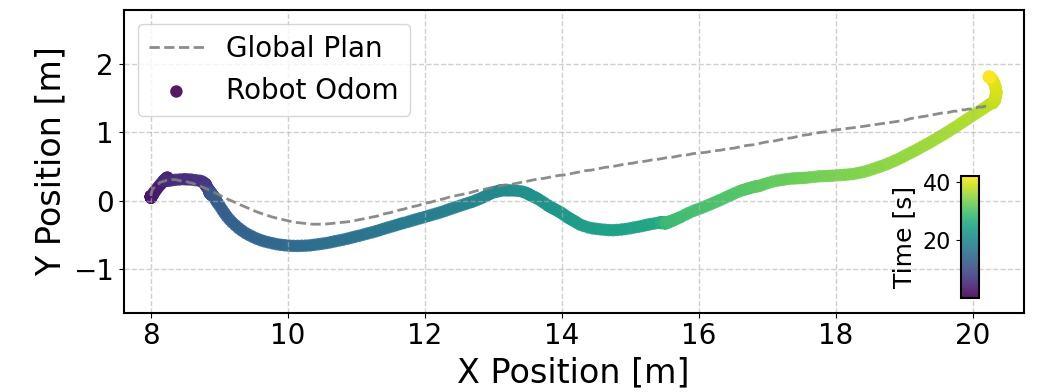}\hfill
    \includegraphics[width=0.30\linewidth]{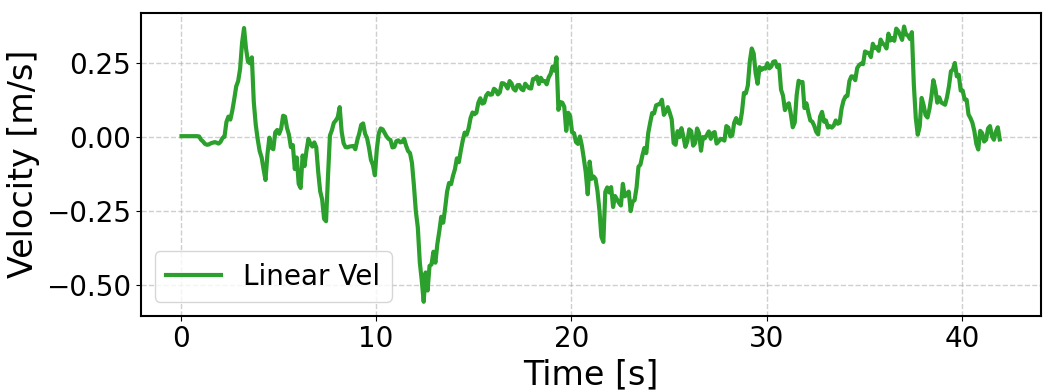}\hfill
    \includegraphics[width=0.30\linewidth]{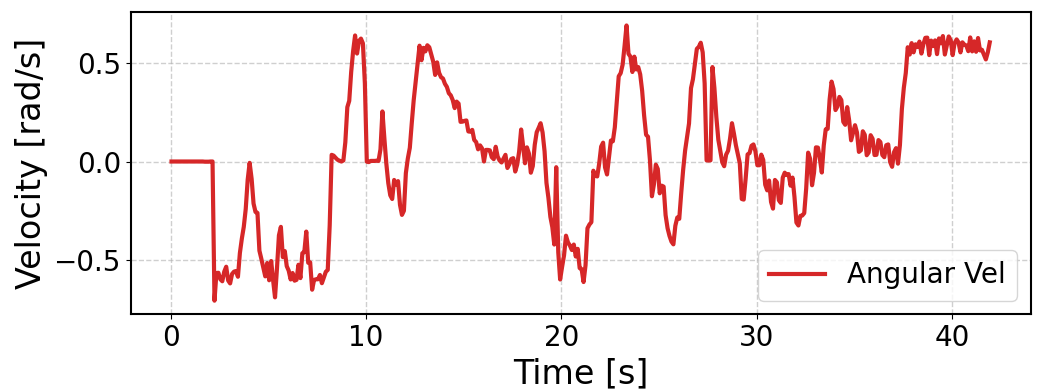}

    \vspace{-1pt}

    \includegraphics[width=0.30\linewidth]{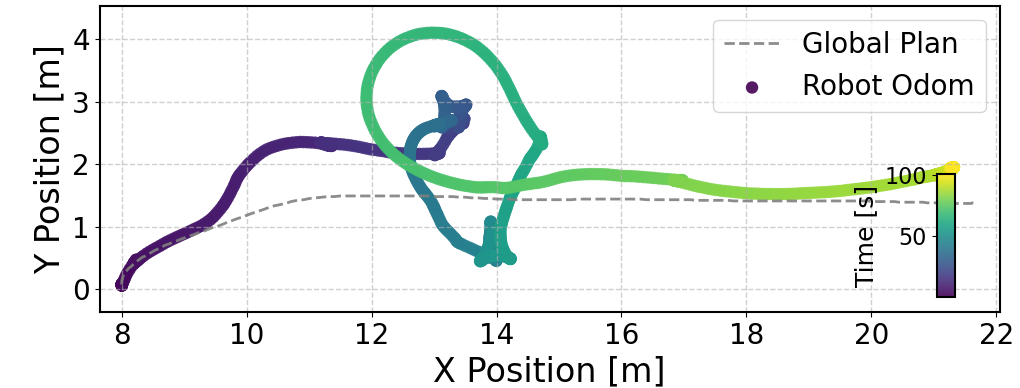}\hfill
    \includegraphics[width=0.30\linewidth]{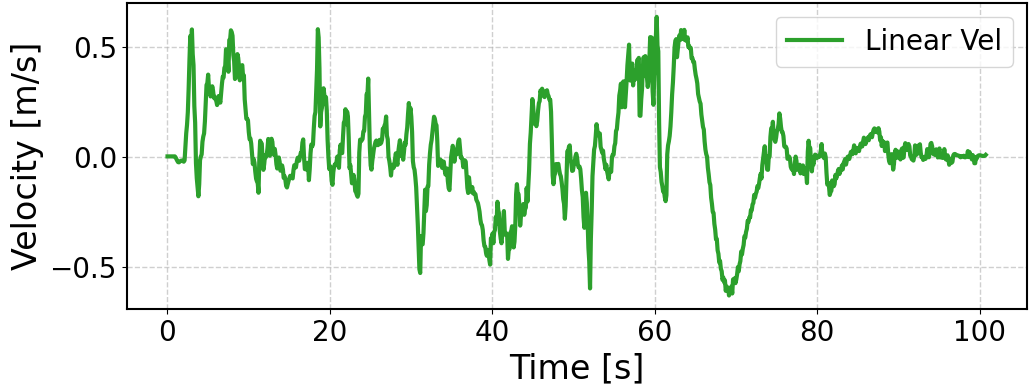}\hfill
    \includegraphics[width=0.30\linewidth]{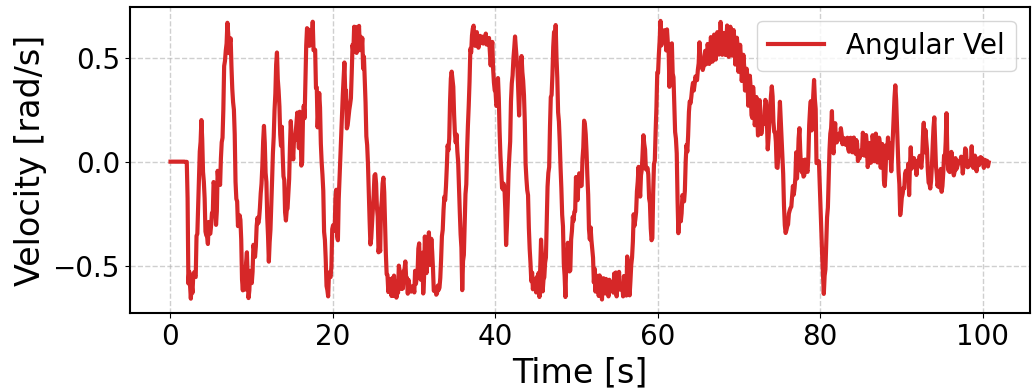}

    \caption{Comparison of robot motion generation between our method (top) and ACP-MPC (bottom): trajectory (left), linear velocity (middle), and angular velocity (right).}
    \label{fig:ours_robot_data}
\end{figure*}

\subsection{Ablation Studies}
\label{sec:ablation_study}
To assess the individual contributions of our framework, we conduct two sets of ablation studies designed to answer the following questions: (1) Can CPI effectively handle distribution shifts? (2) Does the \textit{relax-then-verify} strategy achieve an effective balance between safety and feasibility?

\subsubsection{On CPI calibration}

Figure \ref{fig:ablation_conformal_marginal} plots the per-frame quantification error ($r-\varepsilon$, the gap between the realized nonconformity score and the estimated quantile), and Table \ref{tab:stats_conformal} summarizes its magnitude across horizons. Under the shift, ACP repeatedly diverges, saturating to the infinite-uncertainty bound (top ticks of the middle panel); we thus report ACP$^*$ over the finite frames only, which still understates its true error. The two bounded schemes are nearly identical at $h{=}1$, but CPI adapts better as the horizon grows, attaining a tighter error at $h{=}12$ ($0.463$ vs.\ $0.543$), where its curve stays centered near zero while Conformal P-control drifts to larger margins. 
\begin{figure}[t]
    \centering
    \includegraphics[width=\linewidth]{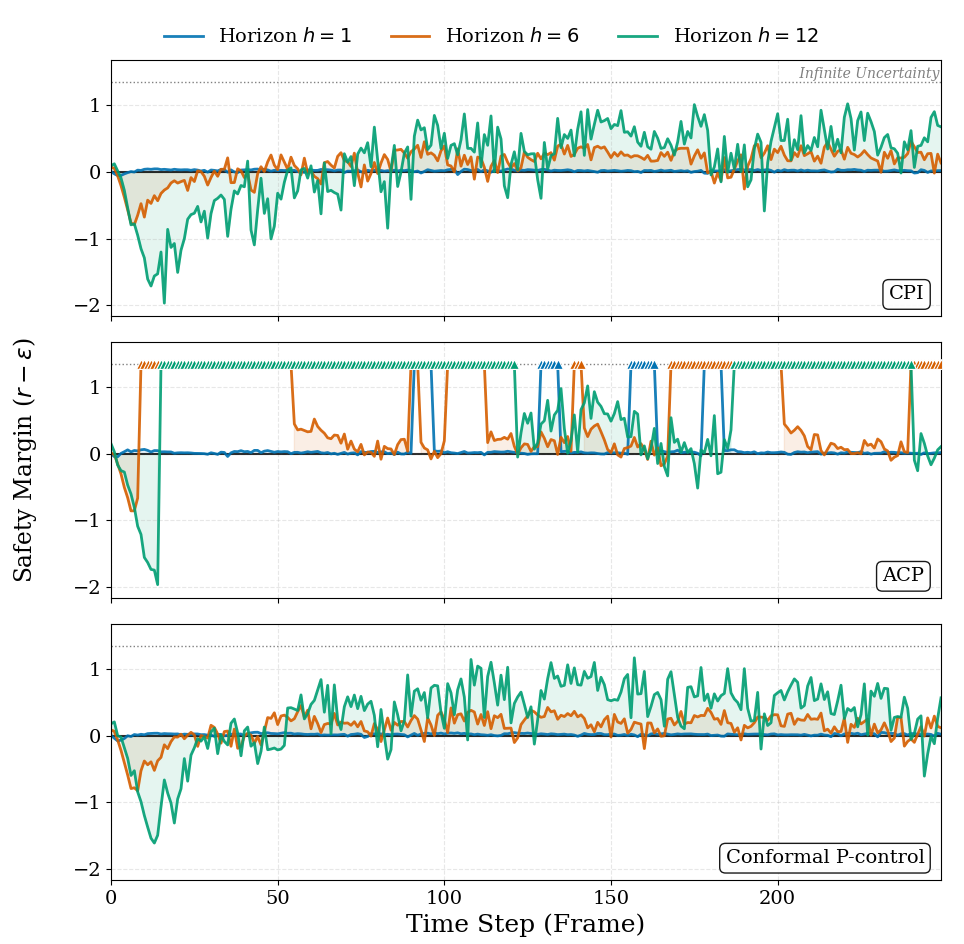}
    \caption{Ablation of conformal safety margins under distribution shift. The margin is the actual nonconformity score minus the estimated quantile.}
    \label{fig:ablation_conformal_marginal}
\end{figure}
\begin{table}[t]
\centering
\small
\caption{Statistical results of conformal safety marginals.}
\label{tab:stats_conformal}
\resizebox{\columnwidth}{!}{%
\begin{tabular}{@{}llll@{}}
\toprule
\multirow{2}{*}{\textbf{Method}} & \multicolumn{3}{c}{\textbf{Horizon}}                                         \\ \cmidrule(l){2-4} 
                                 & \multicolumn{1}{c}{h=1} & \multicolumn{1}{c}{h=6} & \multicolumn{1}{c}{h=12} \\ \midrule
ACP*                             & $0.022{\pm}0.014$             & $0.183{\pm}0.124$             & $0.386{\pm}0.277$              \\
Conformal P-control                       & $0.022{\pm}0.010$             & $0.205{\pm}0.108$             & $0.543{\pm}0.272$              \\
CPI                              & $0.023{\pm}0.010$             & $0.225{\pm}0.104$             & $0.463{\pm}0.249$              \\ \bottomrule
\end{tabular}%
}
\end{table}
Table \ref{tab:conformal_effects} shows the downstream effect on local planning (Unstructured Flow, Ped=8). ACP's divergence cripples its planner: the oversized sets repeatedly render the optimization infeasible, limiting ACP-MPC to a $20\%$ SR at an inflated $58.64$s navigation time. In contrast, both Conformal P-control MPC and CPI-MPC reach a $100\%$ SR, confirming that bounded, well-calibrated uncertainty is essential for reliable planning under domain shift, with CPI-MPC being slightly more efficient ($28.83$s, $13.90$m).
\begin{table}[t]
\centering
\caption{Crowd navigation performance with different conformal prediction methods (unstructured flow, Ped=8).}
\label{tab:conformal_effects}
\resizebox{\columnwidth}{!}{%
\begin{tabular}{@{}llllll@{}}
\toprule
Method  & SR $\uparrow$ & MD(m) $\uparrow$ & CR $\downarrow$ & NT(s) $\downarrow$ & PL(m) $\downarrow$ \\ \midrule
ACP-MPC       & $20\%$   &  $0.60{\pm}0.18$     & $10\%$    &   $58.64{\pm}7.83$    &  $21.18{\pm}2.35$     \\
Conformal P-control MPC &  $100\%$  &   $0.58{\pm}0.08$    & $20\%$   &   $29.50{\pm}1.80$    &   $15.20{\pm}0.90$    \\
CPI-MPC   &  $100\%$  &   $0.56{\pm}0.04$    &  $20\%$  &  $28.83{\pm}1.48$     &   $13.90{\pm}0.42$   \\ \bottomrule
\end{tabular}%
}
\end{table}
\subsubsection{On relax-then-verify control}
Table \ref{tab:verify_module} isolates our two-stage control mechanism in a medium-density scenario (Ped=8), comparing against Quasi-Hard MPC, which enforces the constraints via high cost penalties, and Soft-Only MPC, which relaxes them but omits the posterior verification. Enforcing hard constraints leaves Quasi-Hard MPC over-conservative and erratic, nearly doubling the navigation time ($46.82$s) with a large variance ($\pm13.91$). Relaxation alone (Soft-Only MPC) restores efficiency ($26.75$s) but yields the lowest safety clearance ($0.52$m). Reinstating the verification step, our full scheme recovers the clearance to $0.56$m at a marginal efficiency cost ($28.83$s), while attaining the smallest variance on every metric. These results confirm that the \textit{relax-then-verify} strategy effectively balances safety and feasibility.
\begin{table}[t!]
\centering
\caption{Ablation results of \textit{relax-then-verify} strategy.}
\label{tab:verify_module}
\resizebox{\columnwidth}{!}{%
\begin{tabular}{@{}llllll@{}}
\toprule
\multirow{2}{*}{\textbf{Method}} & \multicolumn{2}{c}{\textbf{Components}} & \multicolumn{3}{c}{\textbf{Metrics}} \\ \cmidrule(l){2-6} 
               & Relaxed? & Verify? & MD(m) $\uparrow$ & NT(s) $\downarrow$ & PL(m) $\downarrow$ \\ \midrule
Quasi-Hard MPC & \ding{55}       & \ding{55}      & $0.65{\pm}0.28$    & $46.82{\pm}13.91$    & $19.03{\pm}3.44$    \\
Soft-Only MPC  & \ding{51}      & \ding{55}      & $0.52{\pm}0.08$    & $26.75{\pm}1.59$    & $13.35{\pm}0.38$    \\
CPI-MPC    & \ding{51}      & \ding{51}     & $0.56{\pm}0.04$    & $28.83{\pm}1.48$    & $13.90{\pm}0.42$    \\ \bottomrule
\end{tabular}%
}
\end{table}
\subsection{Real-World Deployment}
\label{sec:real_exps}
We deploy our approach on a 12-DoF quadruped robot equipped with a Mid-360 LiDAR and an IMX415 camera module. The perception module fuses these two sensing modalities: the LiDAR provides odometry and a real-time occupancy map of static obstacles, while the camera stream is processed by YOLOv11 \cite{khanam2024yolov11} for pedestrian detection and ByteTrackV2 \cite{zhang2022bytetrack} for multi-object tracking. The resulting detections are back-projected with the LiDAR depth and transformed into the global frame to recover each pedestrian's position and track history, which are then fed to the trajectory predictor. This perception module runs on an RK3588-based processing board (4\(\times\) Cortex-A76 cores at 2.4GHz paired with 4\(\times\) Cortex-A55 cores at 1.8GHz) at 10Hz, whereas the prediction and navigation modules run on an NVIDIA Xavier NX board, with the planner operating at 50Hz.
More details are presented in the supplementary video.
\section{CONCLUSIONS}
\label{sec:conclusions}
CoCoNav acts on fallible pedestrian forecasts by keeping online conformal calibration and runtime certification as separate stages, rather than embedding uncertainty sets as hard constraints. This decoupling preserves progress as crowd density grows, sustaining a high success rate with near-zero collisions while avoiding the detours and stalls that make conservative conformal-constrained MPC infeasible. Ablations confirm that horizon-specific PI calibration yields tighter margins under distribution shift and that a~posteriori verification recovers the safety clearance lost by soft-only planning, with onboard quadruped deployment indicating these behaviors transfer to a real robot.
A current limitation is that calibration trusts the prediction stage alone, ignoring perception and state-estimation noise. A promising direction is to fold latent world models into the loop, producing and acting on conformally bounded predictions within a single, tightly integrated decision process.






\bibliographystyle{IEEEtran}
\bibliography{refs}

@article{lindemann2023safe,
  title={Safe planning in dynamic environments using conformal prediction},
  author={Lindemann, Lars and Cleaveland, Matthew and Shim, Gihyun and Pappas, George J},
  journal={IEEE Robotics and Automation Letters},
  volume={8},
  number={8},
  pages={5116--5123},
  year={2023},
  publisher={IEEE}
}

@inproceedings{dixit2023adaptive,
  title={Adaptive conformal prediction for motion planning among dynamic agents},
  author={Dixit, Anushri and Lindemann, Lars and Wei, Skylar X and Cleaveland, Matthew and Pappas, George J and Burdick, Joel W},
  booktitle={Learning for Dynamics and Control Conference},
  pages={300--314},
  year={2023},
  organization={PMLR}
}

@article{gibbs2021adaptive,
  title={Adaptive conformal inference under distribution shift},
  author={Gibbs, Isaac and Candes, Emmanuel},
  journal={Advances in Neural Information Processing Systems},
  volume={34},
  pages={1660--1672},
  year={2021}
}

@article{huang2025interaction,
  title={Interaction-aware Conformal Prediction for Crowd Navigation},
  author={Huang, Zhe and Ji, Tianchen and Zhang, Heling and Pouria, Fatemeh Cheraghi and Driggs-Campbell, Katherine and Dong, Roy},
  journal={arXiv preprint arXiv:2502.06221},
  year={2025}
}

@article{yao2025towards,
  title={Towards Generalizable Safety in Crowd Navigation via Conformal Uncertainty Handling},
  author={Yao, Jianpeng and Zhang, Xiaopan and Xia, Yu and Wang, Zejin and Roy-Chowdhury, Amit K and Li, Jiachen},
  journal={arXiv preprint arXiv:2508.05634},
  year={2025}
}

@article{francis2025principles,
  title={Principles and guidelines for evaluating social robot navigation algorithms},
  author={Francis, Anthony and P{\'e}rez-d’Arpino, Claudia and Li, Chengshu and Xia, Fei and Alahi, Alexandre and Alami, Rachid and Bera, Aniket and Biswas, Abhijat and Biswas, Joydeep and Chandra, Rohan and others},
  journal={ACM Transactions on Human-Robot Interaction},
  volume={14},
  number={2},
  pages={1--65},
  year={2025},
  publisher={ACM New York, NY}
}

@inproceedings{xu2022socialvae,
  title={Socialvae: Human trajectory prediction using timewise latents},
  author={Xu, Pei and Hayet, Jean-Bernard and Karamouzas, Ioannis},
  booktitle={European Conference on Computer Vision},
  pages={511--528},
  year={2022},
  organization={Springer}
}

@article{han2023efficient,
  title={An efficient spatial-temporal trajectory planner for autonomous vehicles in unstructured environments},
  author={Han, Zhichao and Wu, Yuwei and Li, Tong and Zhang, Lu and Pei, Liuao and Xu, Long and Li, Chengyang and Ma, Changjia and Xu, Chao and Shen, Shaojie and others},
  journal={IEEE Transactions on Intelligent Transportation Systems},
  volume={25},
  number={2},
  pages={1797--1814},
  year={2023},
  publisher={IEEE}
}

@article{fox2002dynamic,
  title={The dynamic window approach to collision avoidance},
  author={Fox, Dieter and Burgard, Wolfram and Thrun, Sebastian},
  journal={IEEE robotics \& automation magazine},
  volume={4},
  number={1},
  pages={23--33},
  year={2002},
  publisher={IEEE}
}

@inproceedings{rosmann2017kinodynamic,
  title={Kinodynamic trajectory optimization and control for car-like robots},
  author={R{\"o}smann, Christoph and Hoffmann, Frank and Bertram, Torsten},
  booktitle={2017 IEEE/RSJ International Conference on Intelligent Robots and Systems (IROS)},
  pages={5681--5686},
  year={2017},
  organization={IEEE}
}

@inproceedings{williams2016aggressive,
  title={Aggressive driving with model predictive path integral control},
  author={Williams, Grady and Drews, Paul and Goldfain, Brian and Rehg, James M and Theodorou, Evangelos A},
  booktitle={2016 IEEE international conference on robotics and automation (ICRA)},
  pages={1433--1440},
  year={2016},
  organization={IEEE}
}

@article{rudenko2020human,
  title={Human motion trajectory prediction: A survey},
  author={Rudenko, Andrey and Palmieri, Luigi and Herman, Michael and Kitani, Kris M and Gavrila, Dariu M and Arras, Kai O},
  journal={The International Journal of Robotics Research},
  volume={39},
  number={8},
  pages={895--935},
  year={2020},
  publisher={Sage Publications Sage UK: London, England}
}

@inproceedings{van2011reciprocal,
  title={Reciprocal n-body collision avoidance},
  author={Van Den Berg, Jur and Guy, Stephen J and Lin, Ming and Manocha, Dinesh},
  booktitle={Robotics Research: The 14th International Symposium ISRR},
  pages={3--19},
  year={2011},
  organization={Springer}
}

@article{helbing1995social,
  title={Social force model for pedestrian dynamics},
  author={Helbing, Dirk and Molnar, Peter},
  journal={Physical review E},
  volume={51},
  number={5},
  pages={4282},
  year={1995},
  publisher={APS}
}

@inproceedings{sathyamoorthy2020densecavoid,
  title={Densecavoid: Real-time navigation in dense crowds using anticipatory behaviors},
  author={Sathyamoorthy, Adarsh Jagan and Liang, Jing and Patel, Utsav and Guan, Tianrui and Chandra, Rohan and Manocha, Dinesh},
  booktitle={2020 IEEE International Conference on Robotics and Automation (ICRA)},
  pages={11345--11352},
  year={2020},
  organization={IEEE}
}

@article{liu2022intention,
  title={Intention aware robot crowd navigation with attention-based interaction graph},
  author={Liu, Shuijing and Chang, Peixin and Huang, Zhe and Chakraborty, Neeloy and Hong, Kaiwen and Liang, Weihang and McPherson, D Livingston and Geng, Junyi and Driggs-Campbell, Katherine},
  journal={arXiv preprint arXiv:2203.01821},
  year={2022}
}

@inproceedings{huang2024neural,
  title={Neural informed rrt*: Learning-based path planning with point cloud state representations under admissible ellipsoidal constraints},
  author={Huang, Zhe and Chen, Hongyu and Pohovey, John and Driggs-Campbell, Katherine},
  booktitle={2024 IEEE International Conference on Robotics and Automation (ICRA)},
  pages={8742--8748},
  year={2024},
  organization={IEEE}
}

@inproceedings{alahi2016social,
  title={Social lstm: Human trajectory prediction in crowded spaces},
  author={Alahi, Alexandre and Goel, Kratarth and Ramanathan, Vignesh and Robicquet, Alexandre and Fei-Fei, Li and Savarese, Silvio},
  booktitle={Proceedings of the IEEE conference on computer vision and pattern recognition},
  pages={961--971},
  year={2016}
}

@inproceedings{salzmann2020trajectron,
  title={Trajectron++: Dynamically-feasible trajectory forecasting with heterogeneous data},
  author={Salzmann, Tim and Ivanovic, Boris and Chakravarty, Punarjay and Pavone, Marco},
  booktitle={European Conference on Computer Vision},
  pages={683--700},
  year={2020},
  organization={Springer}
}

@article{nakamura2022online,
  title={Online update of safety assurances using confidence-based predictions},
  author={Nakamura, Kensuke and Bansal, Somil},
  journal={arXiv preprint arXiv:2210.01199},
  year={2022}
}

@article{fridovich2020confidence,
  title={Confidence-aware motion prediction for real-time collision avoidance1},
  author={Fridovich-Keil, David and Bajcsy, Andrea and Fisac, Jaime F and Herbert, Sylvia L and Wang, Steven and Dragan, Anca D and Tomlin, Claire J},
  journal={The International Journal of Robotics Research},
  volume={39},
  number={2-3},
  pages={250--265},
  year={2020},
  publisher={SAGE Publications Sage UK: London, England}
}

@article{fan2021step,
  title={Step: Stochastic traversability evaluation and planning for risk-aware off-road navigation},
  author={Fan, David D and Otsu, Kyohei and Kubo, Yuki and Dixit, Anushri and Burdick, Joel and Agha-Mohammadi, Ali-Akbar},
  journal={arXiv preprint arXiv:2103.02828},
  year={2021}
}

@inproceedings{nair2022stochastic,
  title={Stochastic mpc with multi-modal predictions for traffic intersections},
  author={Nair, Siddharth H and Govindarajan, Vijay and Lin, Theresa and Meissen, Chris and Tseng, H Eric and Borrelli, Francesco},
  booktitle={2022 IEEE 25th International Conference on Intelligent Transportation Systems (ITSC)},
  pages={635--640},
  year={2022},
  organization={IEEE}
}

@article{omainska2021gaussian,
  title={Gaussian process-based visual pursuit control with unknown target motion learning in three dimensions},
  author={Omainska, Marco and Yamauchi, Junya and Beckers, Thomas and Hatanaka, Takeshi and Hirche, Sandra and Fujita, Masayuki},
  journal={SICE Journal of Control, Measurement, and System Integration},
  volume={14},
  number={1},
  pages={116--127},
  year={2021},
  publisher={Taylor \& Francis}
}

@article{angelopoulos2023conformal,
  title={Conformal prediction: A gentle introduction},
  author={Angelopoulos, Anastasios N and Bates, Stephen and others},
  journal={Foundations and trends{\textregistered} in machine learning},
  volume={16},
  number={4},
  pages={494--591},
  year={2023},
  publisher={Now Publishers, Inc.}
}

@article{tibshirani2019conformal,
  title={Conformal prediction under covariate shift},
  author={Tibshirani, Ryan J and Foygel Barber, Rina and Candes, Emmanuel and Ramdas, Aaditya},
  journal={Advances in neural information processing systems},
  volume={32},
  year={2019}
}

@article{strawn2023conformal,
  title={Conformal predictive safety filter for rl controllers in dynamic environments},
  author={Strawn, Kegan J and Ayanian, Nora and Lindemann, Lars},
  journal={IEEE Robotics and Automation Letters},
  volume={8},
  number={11},
  pages={7833--7840},
  year={2023},
  publisher={IEEE}
}

@inproceedings{huang2024conformal,
  title={Conformal policy learning for sensorimotor control under distribution shifts},
  author={Huang, Huang and Sharma, Satvik and Loquercio, Antonio and Angelopoulos, Anastasios and Goldberg, Ken and Malik, Jitendra},
  booktitle={2024 IEEE International Conference on Robotics and Automation (ICRA)},
  pages={16285--16291},
  year={2024},
  organization={IEEE}
}

@inproceedings{lekeufack2024conformal,
  title={Conformal decision theory: Safe autonomous decisions from imperfect predictions},
  author={Lekeufack, Jordan and Angelopoulos, Anastasios N and Bajcsy, Andrea and Jordan, Michael I and Malik, Jitendra},
  booktitle={2024 IEEE International Conference on Robotics and Automation (ICRA)},
  pages={11668--11675},
  year={2024},
  organization={IEEE}
}

@article{mirsky2024conflict,
  title={Conflict avoidance in social navigation—a survey},
  author={Mirsky, Reuth and Xiao, Xuesu and Hart, Justin and Stone, Peter},
  journal={ACM Transactions on Human-Robot Interaction},
  volume={13},
  number={1},
  pages={1--36},
  year={2024},
  publisher={ACM New York, NY}
}

@article{de2024topology,
  title={Topology-driven parallel trajectory optimization in dynamic environments},
  author={De Groot, Oscar and Ferranti, Laura and Gavrila, Dariu M and Alonso-Mora, Javier},
  journal={IEEE Transactions on Robotics},
  year={2024},
  publisher={IEEE}
}

@article{brito2019model,
  title={Model predictive contouring control for collision avoidance in unstructured dynamic environments},
  author={Brito, Bruno and Floor, Boaz and Ferranti, Laura and Alonso-Mora, Javier},
  journal={IEEE Robotics and Automation Letters},
  volume={4},
  number={4},
  pages={4459--4466},
  year={2019},
  publisher={IEEE}
}

@inproceedings{poddar2023crowd,
  title={From crowd motion prediction to robot navigation in crowds},
  author={Poddar, Sriyash and Mavrogiannis, Christoforos and Srinivasa, Siddhartha S},
  booktitle={2023 IEEE/RSJ International Conference on Intelligent Robots and Systems (IROS)},
  pages={6765--6772},
  year={2023},
  organization={IEEE}
}

@article{martinez2024shine,
  title={SHINE: Social homology identification for navigation in crowded environments},
  author={Martinez-Baselga, Diego and de Groot, Oscar and Knoedler, Luzia and Riazuelo, Luis and Alonso-Mora, Javier and Montano, Luis},
  journal={The International Journal of Robotics Research},
  pages={02783649251344639},
  year={2024},
  publisher={SAGE Publications Sage UK: London, England}
}

@inproceedings{xie2021towards,
  title={Towards safe navigation through crowded dynamic environments},
  author={Xie, Zhanteng and Xin, Pujie and Dames, Philip},
  booktitle={2021 IEEE/RSJ International Conference on Intelligent Robots and Systems (IROS)},
  pages={4934--4940},
  year={2021},
  organization={IEEE}
}

@inproceedings{pokle2019deep,
  title={Deep local trajectory replanning and control for robot navigation},
  author={Pokle, Ashwini and Mart{\'\i}n-Mart{\'\i}n, Roberto and Goebel, Patrick and Chow, Vincent and Ewald, Hans M and Yang, Junwei and Wang, Zhenkai and Sadeghian, Amir and Sadigh, Dorsa and Savarese, Silvio and others},
  booktitle={2019 international conference on robotics and automation (ICRA)},
  pages={5815--5822},
  year={2019},
  organization={IEEE}
}

@article{karnan2022socially,
  title={Socially compliant navigation dataset (scand): A large-scale dataset of demonstrations for social navigation},
  author={Karnan, Haresh and Nair, Anirudh and Xiao, Xuesu and Warnell, Garrett and Pirk, S{\"o}ren and Toshev, Alexander and Hart, Justin and Biswas, Joydeep and Stone, Peter},
  journal={IEEE Robotics and Automation Letters},
  volume={7},
  number={4},
  pages={11807--11814},
  year={2022},
  publisher={IEEE}
}

@article{chandra2024towards,
  title={Towards imitation learning in real world unstructured social mini-games in pedestrian crowds},
  author={Chandra, Rohan and Karnan, Haresh and Mehr, Negar and Stone, Peter and Biswas, Joydeep},
  journal={arXiv e-prints},
  pages={arXiv--2405},
  year={2024}
}

@inproceedings{chen2019crowd,
  title={Crowd-robot interaction: Crowd-aware robot navigation with attention-based deep reinforcement learning},
  author={Chen, Changan and Liu, Yuejiang and Kreiss, Sven and Alahi, Alexandre},
  booktitle={2019 international conference on robotics and automation (ICRA)},
  pages={6015--6022},
  year={2019},
  organization={IEEE}
}

@article{everett2021collision,
  title={Collision avoidance in pedestrian-rich environments with deep reinforcement learning},
  author={Everett, Michael and Chen, Yu Fan and How, Jonathan P},
  journal={Ieee Access},
  volume={9},
  pages={10357--10377},
  year={2021},
  publisher={IEEE}
}

@article{khanam2024yolov11,
  title={Yolov11: An overview of the key architectural enhancements},
  author={Khanam, Rahima and Hussain, Muhammad},
  journal={arXiv preprint arXiv:2410.17725},
  year={2024}
}

@inproceedings{zhang2022bytetrack,
  title={Bytetrack: Multi-object tracking by associating every detection box},
  author={Zhang, Yifu and Sun, Peize and Jiang, Yi and Yu, Dongdong and Weng, Fucheng and Yuan, Zehuan and Luo, Ping and Liu, Wenyu and Wang, Xinggang},
  booktitle={European conference on computer vision},
  pages={1--21},
  year={2022},
  organization={Springer}
}

@software{Agarwal_Ceres_Solver_2022,
  author = {Agarwal, Sameer and Mierle, Keir and The Ceres Solver Team},
  title = {{Ceres Solver}},
  license = {Apache-2.0},
  url = {https://github.com/ceres-solver/ceres-solver},
  version = {2.2},
  year = {2023},
  month = {10}
}

@article{angelopoulos2023conformalpid,
  title={Conformal pid control for time series prediction},
  author={Angelopoulos, Anastasios and Candes, Emmanuel and Tibshirani, Ryan J},
  journal={Advances in neural information processing systems},
  volume={36},
  pages={23047--23074},
  year={2023}
}

@article{hsu2023safety,
  title={The safety filter: A unified view of safety-critical control in autonomous systems},
  author={Hsu, Kai-Chieh and Hu, Haimin and Fisac, Jaime F},
  journal={Annual Review of Control, Robotics, and Autonomous Systems},
  volume={7},
  year={2023},
  publisher={Annual Reviews}
}

@inproceedings{bastani2021safe,
  title={Safe reinforcement learning with nonlinear dynamics via model predictive shielding},
  author={Bastani, Osbert},
  booktitle={2021 American control conference (ACC)},
  pages={3488--3494},
  year={2021},
  organization={IEEE}
}

@article{wabersich2021predictive,
  title={A predictive safety filter for learning-based control of constrained nonlinear dynamical systems},
  author={Wabersich, Kim Peter and Zeilinger, Melanie N},
  journal={Automatica},
  volume={129},
  pages={109597},
  year={2021},
  publisher={Elsevier}
}

@article{wabersich2021probabilistic,
  title={Probabilistic model predictive safety certification for learning-based control},
  author={Wabersich, Kim P and Hewing, Lukas and Carron, Andrea and Zeilinger, Melanie N},
  journal={IEEE Transactions on Automatic Control},
  volume={67},
  number={1},
  pages={176--188},
  year={2021},
  publisher={IEEE}
}

\end{document}